\documentclass[lettersize,journal]{IEEEtran}
\usepackage[utf8]{inputenc}
\usepackage{algpseudocode}
\usepackage{algorithm}
\usepackage{url}

\usepackage{amsmath, amsfonts, amssymb, amsthm}
\usepackage{breqn}
\usepackage[subpreambles=false]{standalone}
\usepackage{import}
\usepackage{subcaption}
\usepackage[font=small,labelfont=bf]{caption}
\usepackage{tabularx}
\usepackage{ifthen}
\usepackage{hyperref}
\usepackage{anyfontsize}
\usepackage{tikz}
\usepackage[inline]{enumitem}
\usepackage{hyperref}
\usepackage[noabbrev,capitalize]{cleveref}

\usepackage{multirow}
\usetikzlibrary{decorations.pathreplacing}
\usepackage{booktabs}

\usepackage{siunitx}
\newcommand\restr[2]{{
  \left.\kern-\uninitializeddelimiterspace 
  #1 
  \right|_{#2} 
  }}

\newtheorem{definition}{Definition}[section]
\newtheorem{observation}[definition]{Observation}
\newtheorem{proposition}[definition]{Proposition}
\newtheorem{conjecture}[definition]{Conjecture}

\newtheorem{lemma}[definition]{Lemma}

\newtheorem{corollary}[definition]{Corollary}

\usepackage{xcolor} 
\newcommand{\revisionhighlight}[1]{{#1}}

\title{Time, Travel, and Energy in the Uniform Dispersion Problem}

\author{Michael~Amir%
\thanks{M. Amir is with the University of Cambridge. Email: \href{mailto:ma2151@cam.ac.uk}{ma2151@cam.ac.uk}.}
and Alfred~M.~Bruckstein%
\thanks{A. M. Bruckstein is with the Technion - Israel Institute of Technology. Email: \href{mailto:freddy@cs.technion.ac.il}{freddy@cs.technion.ac.il}.}}

\date{\today}

\begin{document}

\maketitle

\section{Introduction}

\begin{abstract}
We investigate the algorithmic problem of uniformly dispersing a swarm of robots in an unknown, gridlike environment. In this setting, our goal is to study the relationships between performance metrics and robot capabilities. We introduce a formal model comparing dispersion algorithms based on makespan, traveled distance, energy consumption, sensing, communication, and memory. Using this framework, we classify uniform dispersion algorithms according to their capability requirements and performance. We prove that while makespan and travel can be minimized in all environments, energy cannot, if the swarm's sensing range is bounded. In contrast, we show that energy can be minimized by ``ant-like'' robots in synchronous settings and asymptotically minimized in asynchronous settings, provided the environment is topologically simply connected, by using our ``Find-Corner Depth-First Search'' (FCDFS) algorithm. Our theoretical and experimental results show that FCDFS significantly outperforms known algorithms. Our findings reveal key limitations in designing swarm robotics systems for unknown environments, emphasizing the role of topology in energy-efficient dispersion.
\end{abstract}

\begin{IEEEkeywords}
Swarm robotics, energy use, autonomous agents, co-design, formal analysis
\end{IEEEkeywords}

In hazardous unknown environments such as collapsed buildings or leaking chemical factories, the deployment of human teams for mapping, search and rescue, or data gathering can be ineffective and dangerous. In recent years, swarm robotics, which involves the use of numerous simple robots working together to achieve a common goal, has emerged as a promising solution for exploring and monitoring such  environments. The promise of swarm robotics is that a large number of cheap, expendable robots guided by local sensing-based algorithms and   signals can adapt to such environments on the fly, offering an alternative to human intervention \cite{csahin2004swarm, quattrini2020explorationmultirobotrecenttrends}.

A notable approach to using swarms for exploration and monitoring is the concept of ``flooding" or ``uniform dispersion" \cite{hsiang, rappel_area_2019, barrameda2013uniform, hideg2017uniform,  amir_fast_2019,amir_rappel2023stigmergy,barraswarm1,flocchini_uniform_2014}. This approach involves gradually deploying a swarm of small, expendable robotic agents at a fixed entry point in the environment and allowing them to explore and settle until the area is fully covered. By operating autonomously and in parallel, these robots can efficiently gather crucial data while reducing the risk to human personnel \cite{amir_rappel2023stigmergy}. Potential applications range from search and rescue missions and environmental monitoring to scientific exploration. 

In this work, we study the algorithmic problem at the heart of such an approach: flooding an unknown, discrete grid-like environment with a swarm of robots that gradually enter from a source location. This problem, known in the literature as the \textit{uniform dispersion problem}, was formally posed by Hsiang et al. \cite{hsiang} (building on the earlier work of Howard et al. \cite{howard2002}) and has attracted significant attention in the literature--see ``Related Work''. Effective uniform dispersion algorithms are the critical algorithmic backbone of robotic flooding.

\paragraph*{\textbf{Performance Metrics and Robot Capabilities}} A good uniform dispersion algorithm must try to optimize one or more key metrics:


\begin{enumerate}
    \item Makespan: The time it takes for the swarm to fully cover the environment.
    \item Travel: The distance traveled by each robot, which impacts, e.g., wear on the robots.
    \item Energy: The energy consumed by each robot during the exploration process. We  formally measure a robot's energy use as the amount of time it is active in the environment. Energy is of primary importance in robotics, due to battery life and cost considerations.
\end{enumerate}


Since cheap and expendable robots are central to swarm robotics, in addition to these performance metrics, a good uniform dispersion algorithm should strive to minimize the hardware requirements for successful execution, such as sensing range, communication bandwidth (using local signals), and memory or computation capabilities \cite{minimizing1,minimizing2,minimizing3,arxivminimizingtravel}. 

\paragraph*{\textbf{Existing Solutions and Their Limitations}} Prior work on uniform dispersion has made significant progress in several areas:

\begin{enumerate}
\item Makespan optimization: Hsiang et al. \cite{hsiang} demonstrated algorithms that achieve optimal makespan using simple robots with constant sensing range and memory.

\item Robustness: Researchers have developed algorithms that guarantee uniform dispersion under challenging conditions like asynchronicity, reduced memory, myopic sensing, and sudden robot crashes \cite{fekete_deployment_2008, barrameda2013uniform, hideg2016area, hideg2022improved, hideg2017uniform, amir_fast_2019, rappel_area_2019, amir_rappel2023stigmergy}.

\item Different settings and robot dynamics: Various approaches have been proposed for both continuous \cite{stirling_energy-efficient_2010, howard2002, cortes_coverage_2004, arslan2019statisticalcoveragecontrolofmobilesensornetworks} and discrete environments \cite{fekete_deployment_2008,amir_rappel2023stigmergy}, with varying assumptions about the robots.
\end{enumerate}

However, existing solutions have several key limitations. Most prior work has focused on optimizing individual metrics (typically makespan) without considering multiple performance measures simultaneously. Moreover, the relationship between robot capabilities and achievable performance has not been systematically studied across many metrics. Energy efficiency, while crucial for practical applications \cite{aznar_energy-efficient_2018, martinoli_energy-time_2013}, has received particularly limited attention. Additionally, the impact of environment topology on achievable performance remains poorly understood, leaving open questions about how environmental constraints affect optimization possibilities.

These limitations point to a significant research gap: the lack of a comprehensive theoretical framework for understanding how robot capabilities affect their ability to optimize different performance metrics in uniform dispersion. Specifically:

\begin{itemize}
\item We lack formal methods for comparing uniform dispersion algorithms based on both their capability requirements and performance characteristics.

\item The fundamental relationships between makespan, travel, and energy optimization remain unclear.

\item The minimal robot capabilities needed to achieve different optimization goals have not been systematically determined.

\item The role of environment topology in enabling or constraining performance optimization is not well understood.
\end{itemize}

This gap has practical implications for swarm robotics system design, as it leaves open questions about what level of robot sophistication is truly necessary for different performance objectives.

\paragraph*{\textbf{Research Questions and Approach}} To address these gaps, we investigate the following key questions:

\begin{itemize}
\item What is the connection between minimizing makespan, energy, and travel? Under what conditions can we minimize all of these metrics simultaneously?

\item Given a robot with specific capabilities, which of the following metrics can it minimize: makespan, total distance traveled, or energy consumption? How do the robot's capabilities influence its ability to optimize these performance measures?
\end{itemize}

We analyze these questions from a formal, mathematical lens. We introduce a capability-based model for comparing uniform dispersion algorithms and use it to prove fundamental results about which performance metrics can be optimized by robots with different capability levels. This theoretical framework reveals fundamental relationships between robot capabilities, environment topology, and achievable performance, with notable implications for swarm robotics system design.

Specifically, we compare uniform dispersion algorithms in terms of three performance metrics: makespan, travel, and energy, and three capability requirements:  sensing, communication, and persistent state memory (a proxy for the dispersion strategy's complexity). Robots with sensing range $V$ (how many grid cells they can see around them), communication bandwidth $B$ (how many bits they can broadcast per time step), and persistent state memory $S$ (how many bits of memory persist between time steps) are described as $(V,B,S)$-robots. For example, a $(2,0,5)$-robot can see two cells in each direction, cannot communicate, and maintains 5 bits of state memory. With this framework, we identify sufficient robot capabilities for attaining optimal performance in each performance metric. Full details of the robot capability model are provided in Section~\ref{section:model}. The results are summarized in \cref{table:taxonomy}.

\begin{table}[t]
\centering
\caption{Summary of uniform dispersion results by performance metric and environment type.}
\label{table:taxonomy}
\resizebox{\columnwidth}{!}{%
\begin{tabular}{@{}c c l@{}}
\toprule
\textbf{Metric} & \textbf{Environment} & \textbf{Can be minimized by} \\
\midrule
Makespan & General & (2, $\mathcal{O}(1)$, $\mathcal{O}(1)$)-robots \\
& & (Hsiang et al. \cite{hsiang}) \\
\addlinespace
Travel & General & (2, 1, $\mathcal{O}(n \log n)$)-robots \\
\addlinespace
Energy & General & Not possible assuming constant $V$ \\
\addlinespace
All metrics & Simply Connected & (2, 0, 5)-robots (Synchronously) \\
(simultaneously) & & (2, 1, 5)-robots (Asynchronously) \\
\bottomrule
\end{tabular}%
}
\end{table}

Our results show that, although there exist uniform dispersion algorithms that minimize makespan (a result due to Hsiang et al. \cite{hsiang}) and travel (our own result) in general environments, there does not exist an algorithm that minimizes total energy in general environments, assuming robots do not know the environment in advance (in other words, $(V, \infty, \infty)$-robots cannot minimize energy for any given, constant $V$). In fact, we prove that even a centralized algorithm with unlimited computational power but without prior knowledge of the environment cannot minimize energy (\cref{section:energytravelmakespan}). Informally, this is because an energy-optimal algorithm simultaneously also minimizes both makespan and travel, hence energy necessitates both \textit{deciding quickly} and \textit{optimal path-finding}--two goals in conflict with each other.


Given this conclusion, we are led to ask whether energy can be minimized in more restricted types of environments.  We show that a sufficient topological condition is simply-connectedness \cite{joshi1983introductiontogeneraltopology}. A simply connected environment is one where any closed loop can be continuously shrunk to a point without leaving the environment, i.e., the environment has no enclosed ``holes"  that a loop could get stuck around. Examples of simply connected environments include convex shapes like rectangles, as well as more complex shapes like mazes or building floors, as long as they don't have any enclosed holes. We describe an algorithm, ``Find-Corner Depth-First Search'' (FCDFS), that minimizes energy--thus also makespan and travel--in such environments. \revisionhighlight{The idea of FCDFS is to maintain a geometric invariant: the environment must stay simply connected whenever a robot settles}. FCDFS is an algorithm for $(2,0,5)$-robots, meaning it requires no communication and has small, constant sensing range and memory requirements. The $5$ bits of memory are used to store information about the previous steps of the agent in a highly compressed way, enabling it to exploit aspects of simply connected topologies. FCDFS is presented and studied in  \cref{section:minimizingenergyfcdfs} (\cref{alg:FCDFS}). Specifically,  \cref{appendix:fcdfs5bit} shows it can be implemented with $5$ bits of memory. Such low requirements make it executable by very simple \textit{``ant-like''} robots, in contrast with our impossibility results in non-simply connected environments. 

An ``ant-like'' robot is a robot with small, constant persistent state memory and sensing range, and no communication capabilities \cite{shiloni2010antrobotelephantrobot,shats2023competitive}. Such robots have attracted significant attention due to their simplicity and robustness, as well as the interesting algorithmic challenges they pose (see e.g. \cite{deneubourg1991dynamics}). In \cref{appendix:stronglydecentralizedantrobots}, we show that ant-like robots are ``strongly'' decentralized: they are fundamentally incapable of computing the same decisions as a centralized algorithm (\cref{prop:antlikerobotsarestronglydecentralized}). The core idea of FCDFS is maintaining a ``geometric invariant'': the environment must remain simply connected after a robot settles, blocking part of the environment off. In other words, it is reliant on robots \textit{actively reshaping} the environment to guide other robots. This is an example of a concept from ant-robotics and biology called ``stigmergy'': indirect communication via the environment \cite{amir_rappel2023stigmergy,stirling_energy-efficient_2010,martinoli_energy-time_2013}.

The results we described above all assume a \textit{synchronous} setting where robots all activate simultaneously at every time step, but in real world applications, we expect asynchronicity. It is thus  important to ask whether FCDFS generalizes to asynchronous settings. We show that it does in \cref{section:asynchfcdfs}, where  we consider an asynchronous setting with random robot  activation times, \revisionhighlight{relaxing many of our structural assumptions regarding, e.g., inter-robot distance}. We show that,  asymptotically, an asynchronous variant of FCDFS remains time, travel, and energy-efficient. To prove this, we relate FCDFS to the ``Totally Asymmetric Simple Exclusion Process'' (TASEP) in statistical physics, studied extensively in e.g. \cite{prahofer2002current,tracy2009asymptotics,Johansson2000}. 

\paragraph*{\textbf{Summary of Contributions}} 
\begin{enumerate}
    \item A formal model for comparing uniform dispersion algorithms based on robot capabilities (sensing range, communication bandwidth, persistent memory), which guides swarm robotics design by identifying the minimal capabilities needed to achieve different performance objectives.
    
    \item Fundamental impossibility results showing that while makespan and travel can be minimized individually in all environments by simple robots, energy (which requires minimizing them \textit{simultaneously}) cannot be minimized even by a centralized algorithm with unlimited computational power. This reveals inherent limitations in swarm robotics design: in general environments, perfect energy efficiency is unattainable regardless of robot capabilities.
    
    \item The FCDFS algorithm, which provably achieves energy-optimal dispersion in simply connected environments using minimal robot capabilities (no communication, small sensing range, and only a few bits of memory). This demonstrates that in certain common environments like convex shapes and mazes, energy-efficient dispersion is achievable even with very simple, cheap robots.
    
    \item AsynchFCDFS, an asynchronous variant of FCDFS that provably maintains energy efficiency under timing conditions where robots activate randomly and asynchronously. We provide rigorous performance guarantees that hold even with unreliable robot activation.
    
    \item Theoretical results characterizing the distinction between centralized algorithms, decentralized algorithms, and ``strongly" decentralized ant-like algorithms. These results inform the design choice between sophisticated centralized control and simple decentralized strategies.

    \item In addition to the theoretical performance guarantees for FCDFS and AsynchFCDFS, we present  comprehensive empirical evaluation (\cref{section:experiments}) demonstrating that FCDFS significantly outperforms existing algorithms in the literature (BFLF and DFLF \cite{hsiang}), while AsynchFCDFS maintains performance within constant factors of optimal even with random robot activation. This evaluation affirms our theoretical results, further validating our claims. 
\end{enumerate}

Our results emphasize the importance of co-design in swarm robotics: rather than building increasingly sophisticated robots, designers should carefully match robot capabilities to environmental constraints. We show that in many practical scenarios, very simple robots can achieve optimal performance if the environment permits it.

Furthermore, our results address a question first raised by Hsiang et al. \cite{hsiang} about simultaneously minimizing makespan and robot travel. When considering energy use (which includes time spent stationary while active), we prove such simultaneous optimization is impossible in general environments, but achievable in simply connected environments--even by simple 'ant-like' robots using only $5$ bits of memory and no communication. The case where time spent stationary is not counted remains open.




\subsection{Related Work}

This work builds upon and significantly develops results from \cite{arxivminimizingtravel}, which introduces the FCDFS algorithm, proving its energy optimality in a simply-connected, synchronous setting along with the related impossibility result \cref{prop:nototalenergymincentrallized}. We contextualize the results of \cite{arxivminimizingtravel} by interpreting them within our formal model of robot capabilities, which enables us to compare different uniform dispersion algorithms in a principled way. Expanding on \cite{arxivminimizingtravel}, we prove there exists a formal distinction between optimizing \textit{energy use} and \textit{distance travelled} by showing that the former is impossible but the latter is possible in general environments. We introduce an asynchronous version of FCDFS (``AsynchFCDFS'') and prove efficient performance guarantees. Additionally, as part of our  formal study of energy, travel, and makespan, we prove results related to the formal distinction between centralized algorithms, decentralized algorithms, and ``strongly'' decentralized, ant-like  algorithms - see \cref{section:energytravelmakespan} and \cref{appendix:stronglydecentralizedantrobots}. Note that certain definitions differ somewhat between the two papers (due to our more nuanced model), affecting the way quantities are expressed in some results.

A large body of work exists on deploying robotic swarms for covering, exploring, or monitoring unknown environments \cite{quattrini2020explorationmultirobotrecenttrends}. Decentralized approaches are often favored, as they tend to be more adaptive than centralized control to uncertain environmental conditions, e.g.,  \cite{peleg_distributed_2005,cortes_coverage_2004,tran2021robust,howard2002,hsiang,flocchini_uniform_2014,rappel_area_2019,amir_rappel2023stigmergy,tran2023dynamicfrontiercoverage}. Among these approaches, significant attention has been given to uniform dispersion in particular. Foundational work on uniform dispersion can be traced back to Hsiang et al. \cite{hsiang} and Howard et al. \cite{howard2002}. Substantial subsequent developments focused on provably guaranteeing uniform dispersion under conditions such as asynchronicity, reduced or corrupted memory, myopic sensing capabilities, and adaptation to sudden robot crashes \cite{fekete_deployment_2008,barrameda2013uniform,hideg2016area,hideg2022improved,hideg2017uniform,amir_fast_2019,rappel_area_2019,amir_rappel2023stigmergy}. Works on dispersion and similar problems have either considered environments as continuous  \cite{stirling_energy-efficient_2010,howard2002,cortes_coverage_2004,arslan2019statisticalcoveragecontrolofmobilesensornetworks} or discrete environments  \cite{fekete_deployment_2008,amir_rabonovich_optimally_reordering_mobile_agents}, with our work choosing the latter. \revisionhighlight{As our goal in this work is to characterize the \textit{algorithmic} aspects of makespan, travel, and energy efficiency in uniform dispersion, we use discrete grid-like environments that abstract away continuous motion dynamics and geometric nuances}. \revisionhighlight{This type of abstraction is useful and common in many robotic systems, such as warehouse robotics \cite{huang2021learning}}.

Energy is a primary constraint in robotics, and consequently, many works have dealt with predicting, or optimizing energy use in swarms \cite{aznar_energy-efficient_2018,martinoli_energy-time_2013}. Whereas Hsiang et al. focused on makespan and travelled distance as optimization targets \cite{hsiang}, the first work to consider \textit{activity time} as a proxy for energy use in a uniform dispersion context is \cite{arxivminimizingtravel} (whose results this present work builds upon). Subsequently, in \cite{amir_rappel2023stigmergy}, redundancy (having more robots than necessary) is shown to sometimes enable energy savings in a dual-layer uniform dispersion strategy.

Recent robotics' work on travel and energy use has focused on planning and navigation in diverse contexts, including adaptive path planning for environment monitoring \cite{gomathi2023adaptive}, formation-based cooperative reconnaissance \cite{zhang2023formation}, integrated path planning with speed control \cite{ren2020optimal}, \revisionhighlight{and multi-objective scheduling \cite{Cai2024} and drone-routing \cite{SavuranKarakaya2024}}. \revisionhighlight{Unlike our setting, many of these methods rely on full knowledge of the map, and develop empirical strategies for particular scenarios.} \revisionhighlight{Our work focuses, instead, on \textit{theoretical foundations}:} we establish theoretical bounds on energy optimization in uniform dispersion, proving it impossible in \textit{unknown} environments regardless of robot capabilities. This impossibility result, alongside our formal framework relating robot capabilities to  achievable performance, is complementary to such approaches.

Our investigation of energy-efficient dispersion is grounded in the concept of co-design, which involves studying what kinds of environments can improve robots' performance in a given task \cite{gielis2022criticalcodesign}.


To bound the makespan and energy use of AsynchFCDFS, we connect it to the ``Totally Asymmetric Simple Exclusion Process'' (TASEP), a particle process in statistical physics \cite{prahofer2002current,tracy2009asymptotics,Johansson2000}. While TASEP has been applied to several domains, including traffic flow \cite{chowdhury2000statistical} and biological transport \cite{chou2011biologicaltasep}, \revisionhighlight{to our knowledge--its only prior appearance in swarm robotics is our earlier conference paper \cite{amir_fast_2019}, where a simpler \textit{single-active-agent} model was analyzed and the coupling served solely to obtain a high-probability \emph{makespan} bound. The present work advances that idea in two fundamental ways. First, our asynchronous setting allows \textit{multiple} robots to move concurrently, and we must account for that in our analysis. Second--and most importantly--we show how we can use the connection to TASEP to bound the cumulative \textit{activity time} of every robot, yielding provable \emph{energy} guarantees (\cref{prop:asynchfcdfsmaxenergy}).} Readers interested in the mathematical foundations of TASEP are  recommended the introductory works \cite{romik2015surprising} and \cite{kriecherbauer2010pedestrian}.

\section{Model}
\label{section:model}
In the uniform dispersion problem, a swarm of autonomous robots (also called ``agents'') $A_1, A_2, \ldots$ are tasked with completely filling an \textit{a priori unknown} discrete region $R$, attempting to reach a state such that every location in $R$ contains a robot. We assume $R$ is a subset of size $n$ of the infinite grid $\mathbb{Z}^2 = \mathbb{Z} \times \mathbb{Z}$, and represent its locations via coordinates $(x,y)$ where $x$ and $y$ are both integers.  Two locations $(x_1, y_1)$ and $(x_2, y_2)$ are connected if and only if the Manhattan distance $|x_1 - x_2| + |y_1 - y_2|$ is exactly $1$ (so diagonal locations are not connected). The \textit{complement} of $R$, denoted $R^c$, is defined as the subregion $\mathbb{Z}^2 - R$. We call the vertices of $R^c$ \textit{walls}.

Each robot (or ``agent'') is represented as a mobile point in $R$. Robots are dispersed over time onto $R$ and move within it between connected locations, always occupying a single grid-cell of $R$. Robots are identical and anonymous, are initialized with identical orientation, \revisionhighlight{maintaining a global sense of orientation through dead-reckoning or a compass}, and use the same local algorithm to decide their next movements. \revisionhighlight{A robot should not move to a location already occupied by another robot, nor should two robots move to the same location simultaneously. We do not \textit{assume} such collision avoidance, rather, the algorithms we study are designed and formally proven to satisfy this requirement. }

Note that in this discrete grid model, both robots and walls simply occupy integer coordinates in $\mathbb{Z}^2$, with no notion of continuous shapes or sizes. This discrete representation simplifies both sensing and movement: robots need only detect whether neighboring grid cells are occupied or unoccupied, and movement consists of transitioning between adjacent grid cells.

Time is discretized into time steps $t = 1, 2, \ldots$. At the beginning of every time step, all robots perform a Look-Compute-Move operation sequence. During the ``Look'' phase, robots examine their current environment and any messages broadcast by other robots in the previous time step. In the ``Compute-Move'' phases the robots move to a location computed by their algorithm (or stay in place), and afterwards broadcast a message if they wish. The \textit{beginning of time step $t$} refers to the configuration of robots before the Look-Compute-Move sequence of time $t$, and \textit{the end of time step $t$} refers to the configuration after. We shall say ``at time $t$'' for shorthand when the usage is clear from context.

A unique vertex $s$ in $R$ is designated as the source vertex (sometimes called the ``door'' in the literature \cite{hideg2016area}). If at the beginning of a time step there is no mobile robot at $s$, a new robot emerges at $s$ at the end of that time step. We label the emerging robots $A_1, A_2, A_3, \ldots$ in order of their arrival, such that $A_i$ is the $i$th robot to emerge from $s$. 

All robots are initially active, and eventually \textit{settle}. Settled robots never move from their current position. In other words, a robot ``settles'' after it arrives at its final, desired location. A robot may not move and settle in the same time step.

We wish to study swarms of autonomous robots whose goal is to attain uniform dispersion while minimizing three performance metrics: makespan, travel, and energy.

\begin{definition}
The \textbf{makespan} of an algorithm \textbf{\textrm{ALG}} over $R$, denoted $\mathcal{M}$, is the first time step $t$ such that at the end of $t$, every location contains a settled robot. 
\end{definition}

\begin{definition}
Assuming the robots act according to an algorithm \textbf{\textrm{ALG}}, let $T_i$ be the number of times $A_i$ moves. The \textbf{total travel} of \textbf{\textrm{ALG}} over $R$ is defined as $T_{total} = \sum_{i=1}^n T_i$. The \textbf{maximal individual travel} is $T_{max} = \max_{1 \leq i \leq n}T_i$.
\end{definition}

Whereas travel measures the distance a robot has moved since its arrival, energy measures the amount of time during which it was active and consumed energy. Hence, a robot's \textit{energy use} continues increasing even when when it stays put, as long as it isn't settled \revisionhighlight{(e.g., a drone hovering in place)}:

\begin{definition}
Assuming the robots act according to an algorithm \textbf{\textrm{ALG}}, let $t_i^{start}$ be the time step $A_i$ arrives at $s$, and let $t_i^{end}$ be the time step $A_i$ settles. We define \textbf{the energy use of $A_i$} to be $E_i = t_i^{end} - t_i^{start}$. The \textbf{total energy} of \textbf{\textrm{ALG}} over $R$ is defined as $E_{total} = \sum_{i=1}^n E_i$. The \textbf{maximal individual energy use} is $E_{max} = \max_{1 \leq i \leq n}E_i$.
\end{definition}

In context of this work, where all algorithms we consider have similar, neglibile, computational complexity, the primary energy expenditure differential comes from physical activities like movement, sensing, and maintaining basic operation---activities directly proportional to active time. This makes activity time a suitable proxy for energy consumption, an approach aligned with existing work in swarm robotics \cite{aznar_energy-efficient_2018,martinoli_energy-time_2013}. 

These metrics are designed to capture both worst-case bounds on performance and average performance across the swarm. The maximal metrics ($T_{max}$, $E_{max}$) establish upper bounds on the resource requirements any robot may need, providing guarantees for hardware design and deployment. The total metrics ($T_{total}$, $E_{total}$) measure the swarm's overall resource consumption, while makespan ($\mathcal{M}$) captures the collective completion time. These complementary perspectives---individual bounds and collective costs--are particularly important in swarm applications where both resource guarantees and overall efficiency matter.

\paragraph*{Robot capabilities}

We are interested in studying autonomous robots with limited visibility that do not know their environment in advance. Primarily, we are interested in \textit{simple} robots that have highly limited sensing, computation, and local signaling capabilities (as well as robots that cannot communicate at all). By ``local signalling'' we mean that each robot can broadcast a \textit{local} visual or auditory signal by physical means which is received by all robots that sense it--\revisionhighlight{see for example the light-based scheme in  \cite{rubenstein2012kilobot}}. We assume recipients of a message \revisionhighlight{know the relative position of} the robot that sent the message. \revisionhighlight{Other communication methods, such as radio communication, are also admissible as long as messages can be traced to the robot that originated them; but such schemes are much stronger than what our robots require}. Let us parametrize the robots' capabilities as follows:

\begin{enumerate}
    \item \textbf{Visibility  range ($V$).}  A robot located at $p = (x,y)$  \textit{senses} locations $q \in \mathbb{Z}^2$ at Manhattan distance $V$ or less from $p$. The robot knows  \textit{the position relative to $p$} of every vertex $q \in \mathbb{Z}^2$ it senses (i.e., it knows $q - p$). We assume the robot can tell whether $q$ is $s$ and whether $q$ contains an obstacle (i.e., another robot or a wall). However, it cannot tell from sensing alone whether an obstacle is a wall, a settled robot, or an active robot.
    
    \item \textbf{Communication bandwidth ($B$)} represents how many bits each robot can broadcast during the Move phase of each time step, after computing its next action. The broadcast is received by all robots that sense the broadcasting robot during their next Look phase. When $B=0$, robots cannot broadcast. For example, $B=1$ means a robot can broadcast a single bit (0 or 1) each time step, while $B=8$ allows broadcasting a byte of information.
    
    
    \item  \textbf{Persistent state memory ($S$).} Each robot has a state memory of $S$ bits that persists from time step to time step. At a given time $t$ each robot can be in any of $2^S$ states. 
\end{enumerate}

Note that persistent state memory is distinct from memory use \textit{during} a given time step. In swarm robotics, persistent state memory is often a metric of interest, because it is seen as an indirect measure of the swarm strategy's complexity and ability to recover from errors \cite{barrameda2013uniform,hideg2016area,amir2023patrollingwithabitofmemory}. We do not consider other types of memory use in this work.

In this work, we characterize a robot's capabilities by a 3-tuple $(V,B,S)$. \revisionhighlight{$(V,B,S)$ can be related to our robot's hardware: a powerful platform might have large $V$, $B$, or $S$, whereas a simple robot has small $V$, $B$, and $S$. For example, $V$ is directly related to how powerful the robot's perception hardware is (ranging from the 7 cm IR range on a Kilobot \cite{rubenstein2012kilobot} to pocket-sized LiDARs that see 40 m); and basic perception plus a single on/off LED light gives $B=1$, while for higher $B$ we require a more sophisticated scheme can transmit more bits of information per time step.} 

The parameters $V$, $B$ and $S$ determine the possible algorithms robots can run. A central goal of ours is to study what algorithms become available to the robots as we increase the  values $V$, $B$ or $S$. When can increasing these parameters improve the quality of the robots' dispersion strategy, by enabling them to minimize a given performance metric?

\begin{definition}
\label{definition:xyzrobot}
An $(v,b,s)$-robot is a robot with capabilities $V = v$, $B = b$, $S = s$. A $(v,b,s)$-algorithm is an algorithm that can be executed by $(v,b,s)$-robots. 
\end{definition}

\section{Comparing Makespan,  Travel, and Energy}
\label{section:energytravelmakespan}

The central topic of our work is the relationship between agent capabilities and the agents' ability to complete uniform dispersion while optimizing makespan, travel, or energy. The agent model we presented in \cref{section:model} enables us to formally study uniform dispersion algorithms from these lens. In this section, we show how makespan and travel relate to energy, and study the ability of various types of agents to optimize these performance metrics.

We are mainly interested in the question of how well decentralized $(V,B,S)$-robots can perform compared to ``omniscient'' $(\infty,\infty,\infty)$-robots, i.e.,  robots that have full knowledge of the environment in advance and unlimited capabilities (this can be thought of as the optimal ``offline'' solution). To this end, let us define the performance omniscient robots can attain over an environment $R$:

\begin{definition}
For a given grid environment $R$ with source $s$, denote by $E_{total}^{*}$ the lowest possible total energy consumed by any $(\infty,\infty,\infty)$-algorithm that successfully attains uniform dispersion mission in $R$. Similarly define $E_{max}^{*}$, $T_{total}^{*}$, $T_{max}^{*}$ and $\mathcal{M}^{*}$.
\end{definition}

Let $\textrm{dist}(\cdot, \cdot)$ be the distance between two vertices in $R$. We start our analysis with the following observation:

\begin{observation}
\label{observation:lowerboundparameters}
Let $R$ be an environment with $n$ locations $v_1, v_2, \ldots v_n$ and source $s$.  For \textbf{any} algorithm,
\begin{enumerate}
\item $E_{total} - n \geq T_{total} \geq  \sum_{i=1}^{n}{\textrm{dist}(s,v_i)}$ 
\item $E_{max} - 1 \geq T_{max} \geq \max_{1\leq i \leq n}{\textrm{dist}(s,v_i)}$
\item $\mathcal{M} \geq 2n $
\end{enumerate}
\end{observation}

\begin{proof}
$E_{total} - n \geq T_{total}$ is true because for all $i$, $E_i$ increases whenever $A_i$ moves, and $A_i$ spends one time step in which it does not move becoming settled. $T_{total} \geq \sum_{i=1}^{n}{\textrm{dist}(s,v_i)}$ because every $v_i$ must be reached by some robot that arrived at $s$. By similar  reasoning $E_{max} - 1 \geq T_{max} \geq \max_{1\leq i \leq n}{\textrm{dist}(s,v_i)}$. Furthermore, $\mathcal{M} \geq 2n$, since we require $n$ robots to settle in all vertices, and new robots can only arrive at $s$ in time steps where $s$ is unoccupied, meaning it takes at least $2n - 1$ timesteps for $n$ robots to arrive, and an extra time step for the final arrived robot to settle.
\end{proof}

\begin{proposition}
\label{proposition:omniscientpathoptima}
Let $R$ be an environment with $n$ locations $v_1, v_2, \ldots v_n$ and source $s$. Then:

\begin{enumerate}
\item $E_{total}^{*} - n = T_{total}^{*} = \sum_{i=1}^{n}{\textrm{dist}(s,v_i)}$ 
\item $E_{max}^{*} - 1 = T_{max}^{*} = \max_{1\leq i \leq n}{\textrm{dist}(s,v_i)}$
\item $\mathcal{M}^* = 2n$
\end{enumerate}

Furthermore, these best values are all simultaneously obtained by the same $(\infty,0,0)$-algorithm\footnote{Or, equivalently, an algorithm that requires visibility equal to the current environment's diameter.}.

\end{proposition}

\begin{proof}
Let $R(t)$ be the initial environment $R$ after removing every vertex that is occupied by a \textit{settled} robot at time $t$. The algorithm is simply this: at every time step, each robot takes a step that increases its distance from $s$ in $R(t)$. If there are multiple steps it can take to attain this, it picks the first one in clockwise order. If the robot cannot take any steps to increase its distance from $s$, it settles. We shall  show that this algorithm optimizes all makespan, travel, and energy metrics.

First, note that when enacting this algorithm, two active robots $A_i$, $A_{i-1}$ are always at distance $2$ or more from each other. This is because (by the way agent entrances are handled in our formal model) they must have arrived at least two time steps apart, and $A_i$ has been increasing its distance from $s$ with each of its steps. Consequently, active robots never block each others' paths. It stems from this that the lowest-index active robot $A_i$ will eventually settle in some vertex $v \in R(t)$ that has no neighbors in $R(t)$, $v' \in R(t)$ such that $dist(s,v') > dist(s,v)$. This implies that for any vertex $u \in R(t)$, the distance from $s$ to $u$ in $R(t)$ is necessarily the same as in $R$. Hence $A_i$ arrives at $v$ having taken a shortest path in $R$.

Since, by the above, every robot takes a shortest path to its settling location without breaks, the algorithm's total travel is $\sum_{i=1}^{n}{\textrm{dist}(s,v_i)}$, its total energy is  $n + \sum_{i=1}^{n}{\textrm{dist}(s,v_i)}$, its maximal travel is $\max_{1\leq i \leq n}{\textrm{dist}(s,v_i)}$, and its maximal energy is $1 + \max_{1\leq i \leq n}{\textrm{dist}(s,v_i)}$.

Also by the above, an unsettled agent remains at $s$ for precisely one time step, hence a new agent enters $s$ every two time steps. Hence, after $2n - 1$ time steps, $R$ contains $n$ robots, the last of whom settles at time step $2n$, i.e., $\mathcal{M}^* = 2n$. This concludes the proof. \end{proof}

\cref{proposition:omniscientpathoptima}  shows that robots with complete sensing (or prior knowledge) of the environment can match the lower bounds in \cref{observation:lowerboundparameters}, and therefore $E_{total}^{*}$, $E_{max}^{*}$, $T_{total}^{*}$, $T_{max}^{*}$ and $\mathcal{M}^{*}$ are explicitly known for any environment $R$. We will call an algorithm \textit{optimal} with respect to some performance metric if it matches, when executed over any graph $R$, the value in \cref{proposition:omniscientpathoptima} for that respective parameter. 

Note that an $E_{max}$-optimal algorithm is necessarily $T_{max}$-optimal, and an $E_{total}$-optimal algorithm is necessarily $T_{total}$-optimal. The reverse is not true, because in some algorithms, a robot might stop in place for several time steps despite not being settled. We compare the difficulty of minimizing energy to minimizing travel or makespan in Propositions \ref{prop:makespanminexists}, \ref{prop:travelminexists}, and \ref{prop:nototalenergymincentrallized}. Taken together, these results say that there exist swarm-robotic algorithms that minimize makespan and travel, but no such algorithm exists for minimizing energy. This shows that minimizing energy is fundamentally harder than minimizing travel or makespan. 

\begin{proposition}
\label{prop:makespanminexists}
There is a $(2,\mathcal{O}(1), \mathcal{O}(1))$-algorithm that obtains $\mathcal{M} = \mathcal{M^*}$ in any environment. \cite{hsiang}
\end{proposition}

\begin{proof} 
(Sketch.) This result is due to Hsiang et al. \cite{hsiang}. In \cite{hsiang}, it is shown an algorithm called `Depth-First Leader-Follower'' (DFLF) obtains $2n$ makespan\footnote{The original paper \cite{hsiang} shows $2n-1$ makespan, but does not consider settling as a distinct action requiring another timestep.}. We briefly sketch the main ideas here.

DFLF operates as follows: by default, each robot $A_{i+1}$ chases its predecessor $A_i$ by taking a step along the shortest path to it. The first robot $A_1$ is designated as a `leader-explorer' and attempts, at every time step, to move to an unexplored location (it keeps track of these locations with finite memory by exploiting the property that already-explored location either contains a robot $A_i$, or will contain one in the next time step). When $A_1$ cannot move to an unexplored location it settles, and $A_2$ becomes the new leader-explorer, and so on. This results in the robots exploring $R$ via the edges of an implicit spanning tree, and the analysis of \cite{hsiang} shows this algorithm obtains $2n$ makespan.\footnote{In \cite{hsiang}, DFLF is implemented for finite-automata, anonymized robots with local point-to-point communication. However, their implementation can easily be converted to fit our broadcast communication model as well.}
\end{proof}

\begin{proposition}
\label{prop:travelminexists}
There is a $(2,1,\mathcal{O}(n\log n))$-algorithm that obtains $T_{total} = T_{total}^{*}$ and $T_{max} = T_{max}^{*}$ in any environment consisting of $n$ locations or less.
\end{proposition}

Before proving \cref{prop:travelminexists}, we shall prove a more general claim. A \textit{centralized algorithm} is an algorithm where a central computer sees, simultaneously, what all the robots in the environment see, and issues commands simultaneously to all of them at the beginning of each time step. If the robots have sensing range $V$, such an algorithm can see any location that is at distance $V$ or less from any robot in the swarm, and control the movements of the robots based on this information. By definition, centralized algorithms do not require communication between robots, since a central computer can directly command each robot based on their collective sensory information. 

We wish to show that under certain conditions, decentralized robots can simulate centralized algorithms. For example, let us observe that when the sensing range $V$ is large, centralized algorithms can easily be simulated:

\begin{observation}
Let \textbf{\textrm{ALG}} be a centralized algorithm requiring $\mathcal{S}_R$ persistent state memory and visibility $V$ to run over an environment $R$ of diameter $\mathcal{D}$ (i.e., the distance between any two locations in $R$ is at most $\mathcal{D}$). Then \textbf{\textrm{ALG}} can be simulated in $R$ by decentralized  $(\mathcal{D},0,\mathcal{S}_R)$-robots.
\label{observation:vinftysimulatecentralized}
\end{observation}
 
\begin{proof} When robots see the entire region (i.e., $V = \mathcal{D}$), they all have access to the same information and can simulate \textbf{\textrm{ALG}} independently.
\end{proof}

\cref{observation:vinftysimulatecentralized} says that when robots' sensing range $V$ is large enough, any centralized algorithm can be simulated by a decentralized algorithm. Let us now show that when the robots' memory capacity $S$ is large enough, decentralized robots can simulate a large class of centralized algorithms at the cost of some multiplicative delay in execution time, even when visibility and communication are heavily restricted. 

\begin{definition}
\label{definition:delaydeltadefinition}
If robots executing an algorithm \textbf{\textrm{ALG}} only move during time steps that are multiples of $\Delta$, \textbf{\textrm{ALG}} is said to have \textbf{delay $\Delta$}.
\end{definition}

\cref{definition:delaydeltadefinition} helps us capture the idea of a decentralized algorithm simulating a centralized one at the cost of some multiplicative delay $\Delta$ in execution time. Intuitively, delay $\Delta$ represents the additional time needed for decentralized robots to gather and propagate information that would be immediately available to a centralized algorithm. 

Any centralized algorithm \textbf{\textrm{ALG}} can be slowed down into a $\Delta$-delay algorithm \textbf{\textrm{ALG}'} by simply waiting $\Delta-1$ time steps between each movement command that \textbf{\textrm{ALG}} issues. If a decentralized algorithm simulates \textbf{\textrm{ALG}'}, meaning that at any time step $t$, the decentralized algorithm and \textbf{\textrm{ALG}'} move robots to the exact same positions, we say that it \textbf{simulates \textbf{\textrm{ALG}} at delay $\Delta$}.

\begin{definition}
Given a swarm of robots executing some (centralized or decentralized) algorithm \textbf{\textrm{ALG}} over an environment $R$, we define $G_t^V$ to be a graph whose vertices are all robots that have entered the environment at time $t$ or before and where there is an edge between every two robots $A_i$, $A_j$ whose distance to each other is at most $V$. $G_t^V$ is called the \textbf{$V$-visibility graph} at time $t$.

An algorithm \textbf{\textrm{ALG}} is called \textbf{$V$-visibility preserving} over an environment $R$ if $G_t^V$ is connected at every time step $t$.
\end{definition}

\begin{proposition}
\label{proposition:centralizedsimul}
Let $R$ be a region consisting of $n$ locations. Let \textbf{\textrm{ALG}} be any centralized $V$-visibility preserving algorithm that requires $\mathcal{S}_R$ bits of memory and visibility $V$ when executed over $R$. Then \textbf{\textrm{ALG}} can be simulated by a decentralized $(V,1,\mathcal{O}(n \log n) +\log \Delta +\mathcal{S}_R)$-algorithm at delay $\Delta$ for some $\Delta \geq C V n \log n$ where $C$ is a sufficiently large constant independent of $R$ and \textbf{\textrm{ALG}}.
\end{proposition}

\begin{proof}
We describe a decentralized $(V,1,\mathcal{O}(n \log n) +\log \Delta + \mathcal{S}_R)$-algorithm simulating \textbf{\textrm{ALG}} at delay $\Delta$. 

Let \textbf{\textrm{ALG}'} be the $\Delta$-delay version of \textbf{ALG}. Our robots will operate in \textit{phases}, each consisting of $\Delta$ time steps. The location of a robot $A$ during the entirety of phase $i$ will correspond exactly to its location at the end of time step $\Delta i$ of \textbf{\textrm{ALG}'}. 

Defining $s$ to be located at $(0,0)$, the algorithm has robots keep track of the $(x,y)$ coordinates of themselves and of other robots' positions, as well as the coordinates of obstacles. It works as follows: whenever a robot enters a new phase, it spends $\mathcal{O}(V \log n)$ time steps broadcasting a ``status message" containing its location, and the coordinates of any \textit{previously unreported} obstacles it sees. The message requires $\mathcal{O}(V \log n)$ bits because the perimeter of the set of visible locations for each robot consists of at most $4V$ locations, and when a robot moves, it only needs to report new sightings in the perimeter, since any other location the robot sees at the current time step was visible to it at a previous time step and would have already been reported. Hence, when the robot moves it needs to report the $(x,y)$ coordinates of at most $4V$ obstacles, and each such coordinate can be reported using $\mathcal{O}(\log n)$ bits (as the region consists of $n$ locations thus no coordinate can exceed $n$).  Any robot that receives a reported robot or obstacle position then broadcasts it to its neighbours (once per position). 

By this propagation process, assuming the visibility graph $G_t^V$ remains connected, every robot learns the positions of all the other robots during the current phase, and any new obstacle they see, upon receiving at most $n-1$ status update messages. It receives this information during every phase. Hence, by taking into account the robots' visibility range $V$ and remembering also previously-reported obstacle coordinates, it can know what every robot currently sees. This means it knows everything \textbf{\textrm{ALG}'} sees at any given time step. Since every robot also has the memory capacity required to execute \textbf{\textrm{ALG}'} ($\mathcal{S}_R$), it can simulate its next step according to \textbf{\textrm{ALG}'}. At every time step, each robot maintains an internal counter up to $\Delta$. When this counter reaches $\Delta$, the robot moves according to \textbf{\textrm{ALG}'} and resets the counter to $0$. This counter is there to get our robots to wait a sufficient number of time steps so that they receive all messages sent out during the current phase. Assuming the phase of the algorithm was $i$, a robot resetting the counter to $0$ (note: all robots do this simultaneously) indicates that the algorithm has moved to phase $i+1$. While robots need to count to $\Delta$ within each phase, requiring $\log \Delta$ bits, they need not remember the current phase number, and we save memory by not storing the phase number.


Since there are at most $n$ messages of length $\mathcal{O}(V \log n)$ that need to be propagated, and every robot can transmit one bit per time step, each phase takes at most $\mathcal{O}(V n \log n)$ time steps. Hence, assuming delay $\Delta \geq C \cdot V n \log n$ for a sufficiently large constant $C$, the decentralized robots have sufficient time to receive all messages, and they simultaneously simulate the next step of \textbf{\textrm{ALG}} when the counter reaches $\Delta$.

Keeping track of the counter requires $\log \Delta$ bits. Keeping track of robot and obstacle coordinates requires $\mathcal{O}(n \log n)$ bits. In total, we require $\mathcal{O}(n \log n) +\log \Delta +\mathcal{S}_R$ bits of memory.
\end{proof}

\cref{proposition:centralizedsimul} says that robots with extra memory can simulate any $V$-visibility preserving centralized algorithm, assuming we also slow the algorithm down by a factor of $\Delta = \mathcal{O}(V n \log n)$ (we suspect that this result can be optimized to require smaller $\Delta$, but we do not pursue this here). Note that the class of $V$-visibility preserving algorithms is quite  expressive: for example, it includes the uniform dispersion algorithms of Hsiang et al. and others \cite{hsiang,barrameda2013uniform,hideg2017uniform}, and the uniform dispersion algorithms presented in this work.

Let us now prove \cref{prop:travelminexists}.

\begin{proof}
Let us first describe a centralized algorithm that minimizes travel. The algorithm operates in phases. At each phase it continues expanding a BFS tree of the environment. Specifically, Let $T(i)$ be the set of nodes belonging to the tree at phase $i$. At the beginning of phase $i$, there is a robot at every node of $T(i)$. $T(1)$ contains only the source vertex, $s$. $T(i)$ contains all nodes at distance $i$ from $s$. $T(i+1)$ is obtained by sending a robot to each location at distance $i+1$ from $s$, thereby adding nodes and edges to the tree (the centralized algorithm sees and can compute these locations due to having a robot at every location at distance $i$ from $s$), and also a robot to each location in $T(i)$ that is left empty by its robot leaving it. The robots strictly traverse using only the edges of the BFS tree, hence always take shortest paths to their location (note that this may force a robot to wait in place for several time steps while waiting for a robot to leave the tree node it wishes to move to). Hence, this algorithm obtains optimal travel: the sum total of movements made by the robots is $T_{total}^{*} = \sum_{i=1}^{n}{\textrm{dist}(s,v_i)}$ 
and the maximal individual travel is $T_{max}^{*} = \max_{1\leq i \leq n}{\textrm{dist}(s,v_i)}$.   

This is a $2$-visibility preserving centralized algorithm requiring $\mathcal{O}(n \log n)$ memory, as at each phase we need to store in memory at most $n-1$ edges representing the tree, and each edge connects two locations whose coordinates can be stored using $\mathcal{O}(\log n)$ bits (setting $s$ as $(0,0)$). By \cref{proposition:centralizedsimul}, a decentralized $(2,1,\mathcal{O}(n\log n))$-algorithm therefore exists simulating a $\Delta = \mathcal{O}(n \log n)$-delay version of it.
\end{proof}

Propositions \ref{prop:makespanminexists} and \ref{prop:travelminexists} establish that decentralized swarm-robotic algorithms exist minimizing makespan and travel in all environments. These algorithms assume constant sensing range $V$ and communication bandwidth $B$, and either constant memory capacity $S$ (\cref{prop:makespanminexists}) or memory capacity that scales with the size of the environment (\cref{prop:travelminexists}). Let us now show that, as long as the robots' sensing range is finite, no algorithm exists for minimizing energy in all environments, irrespective of memory capacity and communication bandwidth. In fact, we prove that even a centralized algorithm does not exist for minimizing energy:

\begin{proposition}
\label{prop:nototalenergymincentrallized}
Assuming robots' sensing range $V$ is finite, no centralized algorithm for uniform dispersion is $E_{total}$-optimal in all environments.
\end{proposition}

The proof relies on the construction of a set of environments. Given any algorithm for minimizing energy, \textbf{\textrm{ALG}}, we prove it is sub-optimal in at least one environment in the set.

The environment $G(k,r)$, parametrized by the integers $r$ (a scaling factor) and $k \in [1,10r]$ (denoting a specific column), is depicted in \cref{fig:impossibility}. It contains $10r$ columns of width $1$ spaced $2r$ cells apart,  connected by a bottom row of length $20r^2$. All columns \textit{but the $1$st and $k$th columns} have height $30r^2$. The $1$st and $k$th columns have height $30r^2 + 1$, and are the only columns connected to the top row (see the Figure). The source, $s$, is assumed to be at the bottom left corner.
 
\begin{figure}[!ht]
    \centering
    \includegraphics[height=2.27in]{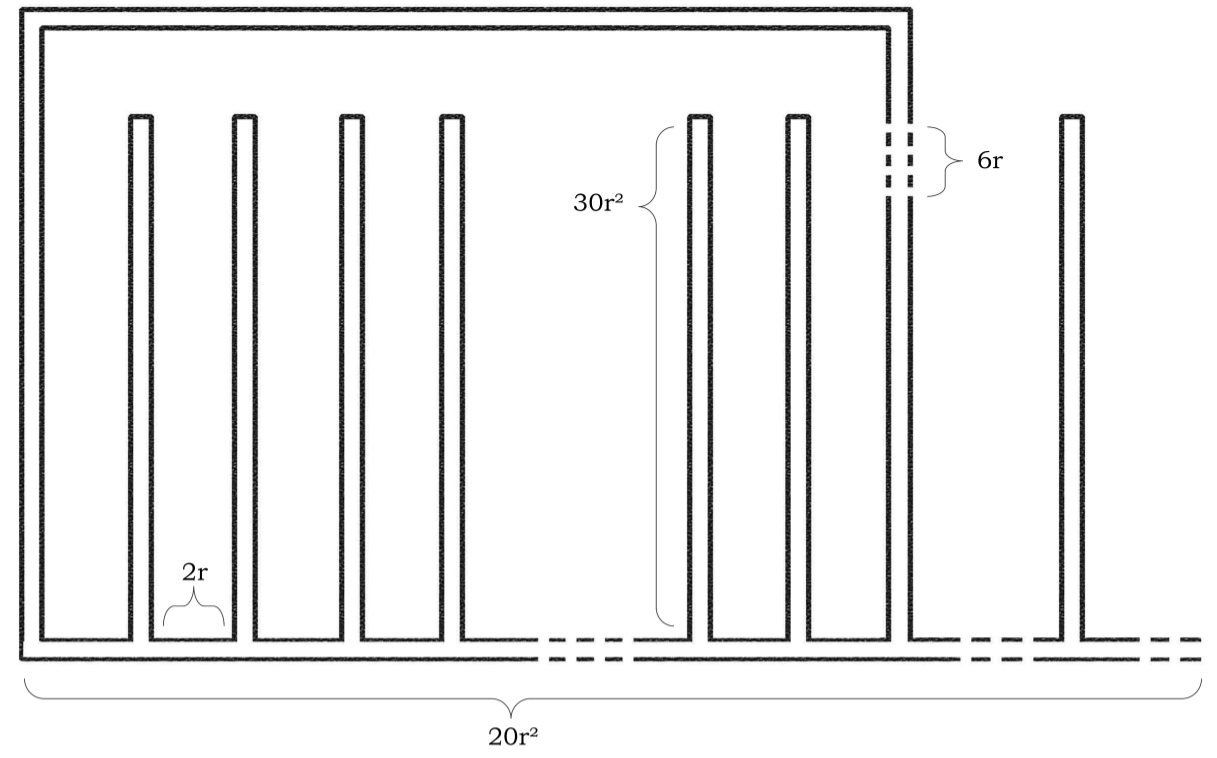}
    \caption{Sketch of the environment $G(k,r)$ (not drawn to scale).}
    \label{fig:impossibility}
\end{figure}

With this construction in mind, let us prove \cref{prop:nototalenergymincentrallized}.

\begin{proof}
Assume for contradiction that $\textbf{\textrm{ALG}}$ is a uniform dispersion algorithm that is total energy-optimal in all environments. We study the actions of \textbf{\textrm{ALG}} over the environment $G(k,V)$. We shall show that there exists $1 \leq k \leq 10V$ for which \textbf{\textrm{ALG}} is sub-optimal over $G(k,V)$, or some rotation and reflection of $G(k,V)$.

At time $t = 0$, $A_1$ is located at the bottom left (since this is where $s$ is), and can move either up or right. Since $A$'s visibility range is $V$, $A$ cannot see any column besides the leftmost column at time $t = 0$. This means that $A$ cannot distinguish between $G(k,V)$ and a rotation-and-reflection of $G(k,V)$. Hence, we can assume w.l.o.g. that \textbf{\textrm{ALG}} makes $A_1$ take a step up (if it steps right, we simply rotate and reflect $G(k,V)$). 


By assumption, the \textit{energy use} of \textbf{\textrm{ALG}} over $G(k,V)$ is $E_{total}^* = n + \sum_{v \in G(k,V)}{\textrm{dist}(s,v)}$. This assumption implies that every robot travels a shortest path to wherever it eventually settles, without rests. In particular, $A_1$ must be active for precisely $\textrm{dist}(A_1, v_1)+1$ time steps, where $v_1$ is the destination at which $A_1$ chooses to settle. 

We make several observations:

\begin{enumerate}
    \item Once $A_1$ stepped up, it has committed to stepping up and right until reaching $v_1$, as staying in place or going in a third direction causes its path to $v_1$ to be longer than $\textrm{dist}(s, v_1)$ steps, causing the total energy use of \textbf{\textrm{ALG}} to be greater than $E_{total}^* $--a contradiction.
    
    \item $v_1$ cannot be a vertex in the first column or in the top row except the top vertex of column $k$ or one vertex to its left. Should $v_1$ not equal one of these two vertices, there will be vertices in the top row that are to the right of $v_1$, but not part of column $k$. Settling at $v_1$ would make it impossible for other robots to reach these vertices by going up the first column and then traveling rightward, forcing these robots instead to reach these vertices by going right until the $k$th column, going to the top of column $k$ and then moving left. This ``U-turn'' is longer than the path that goes up  the first column and then goes right, so this will cause the total energy to increase beyond $E_{total}^*$--a contradiction. 
    
    \item $v_1$ cannot be any vertex in the $k$th column other than the top of the $k$th column, as this would require $A_1$ to step downwards. Stepping downwards results in $A$ taking a sub-optimal ``U-turn'' path to reach its destination--a contradiction.
    
\end{enumerate}

(*) From (1)-(3) we conclude that $v_1$ must equal precisely the top vertex of the $k$th column or one vertex to its left.

Let $T = 30V^2 + 2$ be the time step at the end of which $A_1$ reaches the top row. Since \textbf{\textrm{ALG}} is energy-optimal, no robot can stay at $s$ for more than one time step, so by the time $A_1$ reaches the top row, there will be $4V$ robots in $G(k,V)$ that have been around for at least $T-8V-2$ time steps (energy optimality forces robots to always move away from $s$, so a new robot must emerge at $s$ once per two time steps). Each of these $4V$ robots must have already entered one of the columns or settled, since they travel optimal paths to their destination, and the total length of the bottom row is $20V^2 < T-8V-2$.

Note that at and before time $T$, none of the top vertices of the other columns have been seen, so \textbf{\textrm{ALG}} must follow the course of action outlined above independent of $k$. Note further that  as there are $10V$ columns, there must exist a column, column $k^*$, that none of the robots $A_1, \ldots, A_{4V}$ have entered. Let us set $k = k^*$. We shall show that \textbf{\textrm{ALG}} must act sub-optimally in $G(k^*, V)$. 

When $A_1$ reaches $v_1$, the above indicates that any other robot currently present in the $k^*$th column (if there are any) emerged at $s$ at least $2\cdot 4V$ time steps after $A_1$. Therefore it is at distance at least $8V$ from $A_1$, meaning that there is a segment of $6V$ vertices in column $k^*$ that no robot has seen yet. This indicates that \textbf{\textrm{ALG}} must make the same decision for $A_1$ whether these vertices exist or not. However, if any one of these vertices does not exist, then column $k^*$ is not connected to the top row, indicating that $A_1$ cannot settle at the top of the $k^*$th column or to its left, else it will block off part of the environment. We arrived at a contradiction to (*).\end{proof}

By adding more columns to the $G(k,r)$ construction and increasing the height of the columns, we can force $A_1$ to go down more and more steps, causing the difference between $E_{total}^*$ and the total energy use of \textbf{\textrm{ALG}} to be arbitrarily large.

Our argument does not exclude the possibility of an algorithm that is $E_{max}$-optimal, and we do not know if such an algorithm exists. We leave this as an open problem for interested readers.

\section{Minimizing Energy in Simply Connected Regions}
\label{section:minimizingenergyfcdfs}
We've shown that no swarm algorithm exists that optimizes energy in every environment,  regardless of the robots' capabilities. In this section, we will show that it \text{is} possible to minimize energy in a large class of environments called ``simply connected'' environments:

\begin{definition}
\label{holelessdefinition}
A environment $R$ is said to be  \textbf{simply connected} if $R$ is connected and $R^c$ is connected (i.e., there is a path between any two vertices in $R^c$ consisting only of vertices in $R^c$). 
\end{definition}

Equivalently, $R$ is simply connected if and only if it contains no ``holes'': any path $v_1 v_2 \ldots v_1$ of vertices in $R$ that forms a closed curve does not surround any vertices of $R^c$. 

Surprisingly, not only can we minimize energy in simply connected environments, we can do so with a $(2,0,5)$-algorithm. Such an algorithm can be considered a ``strongly decentralized'', \textit{ant-like algorithm}, i.e., an algorithm requiring no communication, and small, constant visibility range and persistent state memory, to run in its target class of environments (in this case, simply connected environments). We elaborate on the claim that such algorithms are strongly decentralized in \cref{appendix:stronglydecentralizedantrobots}. The main result of this section is an algorithm for $(2,0,5)$-robots, called ``Find-Corner Depth-First Search'' (FCDFS), that enables the robots to disperse over any simply-connected region $R$ while being energy-, travel- and makespan-optimal:

\begin{proposition}
FCDFS is an $E_{total}$, $E_{max}$, $T_{total}$, $T_{max}$,  and $\mathcal{M}$-optimal $(2,0,5)$-algorithm for uniform dispersion in simply connected environments.
\end{proposition}

FCDFS ensures that the path of a robot from $s$ to its eventual destination (the vertex at which it settles) is a shortest path in $R$. The idea of the algorithm lies in the distinction between a \textit{corner} and a \textit{hall} (see Figure \ref{fig:corners} and Figure \ref{fig:halls}):

\begin{definition}
A vertex $v$ of a grid environment $R$ is called a \textbf{corner} if either:
 
\begin{enumerate}[label=(\alph*)]
    \item $v$ has one or zero neighbours in $R$, or
    \item $v$ has precisely two neighbours $u$ and $u'$ in $R$, and $u$ and $u'$ have a common neighbour $w$ that is distinct from $v$.
\end{enumerate}
\end{definition}

\begin{definition}
A vertex $v$ of $R$ is called a \textbf{hall} if it has precisely two neighbours $u$ and $u'$, and $u$ and $u'$ are both adjacent to the same vertex $w$ in $R^c$.
\end{definition}

\begin{figure}[!htb]
  \centering%
    \includegraphics[width=.2\linewidth,scale=0.3]{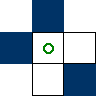}\hfill%
    \includegraphics[width=.2\linewidth,scale=0.3]{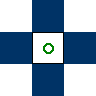}\hfill%
    \includegraphics[width=.2\linewidth,scale=0.3]{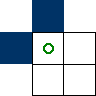}\hfill%
    \caption{Possible corners, denoted by the green circles. White vertices are locations in $R$, and blue vertices are \textit{walls}, i.e., vertices in $R^c$.}
    \label{fig:corners}
\end{figure}

\begin{figure}[!htb]
  \centering%
    \includegraphics[width=.2\linewidth]{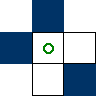}\hfill%
    \caption{A hall, denoted by the green circle..}
    \label{fig:halls}
\end{figure}

Essentially, halls are vertices in $R$ that are blocked by walls on two sides, and have an additional wall $w$ diagonal to them. Corners are either dead-ends, or vertices in $R$ that are blocked by walls on two sides, and have a vertex $w$ of $R$ diagonal to them. If $v$ is either a hall or a corner, $w$ is called the ``diagonal'' of $v$, and is denoted $diag(v)$. We observe that diagonals are uniquely specified.

Robots executing FCDFS attempt to move only in `primary' and `secondary' directions, where the secondary direction is always a 90-degree clockwise rotation of the primary direction (for example "up and right", "right and down", or "down and left"). They may only change their primary direction once they arrive at a hall (\cref{fig:fcdfs_hall_example}, \cref{fig:simulation2}), and they settle once both their primary and secondary directions are blocked and they are at a corner.

\begin{figure}[!ht]
  \centering%
    \includegraphics[width=.49\linewidth]{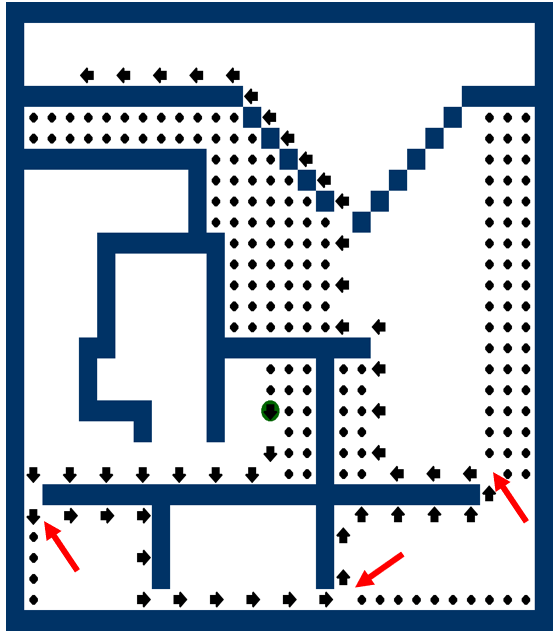}
    \caption{A snapshot from an example run of FCDFS. $s$ is the green circle, black arrows are active robots, black ``dots'' are settled robots, the blue region is $R^c$ and the white region is $R$. The primary direction of each robot is indicated by the direction the arrow representing it is facing \revisionhighlight{(note this is a \textit{state} in the algorithm; robots do not have a ``heading'')}. Red arrows point to some of the halls.}
    \label{fig:fcdfs_hall_example}
\end{figure}

\begin{figure}[!ht]
  \centering%
    \includegraphics[width=.5\linewidth]{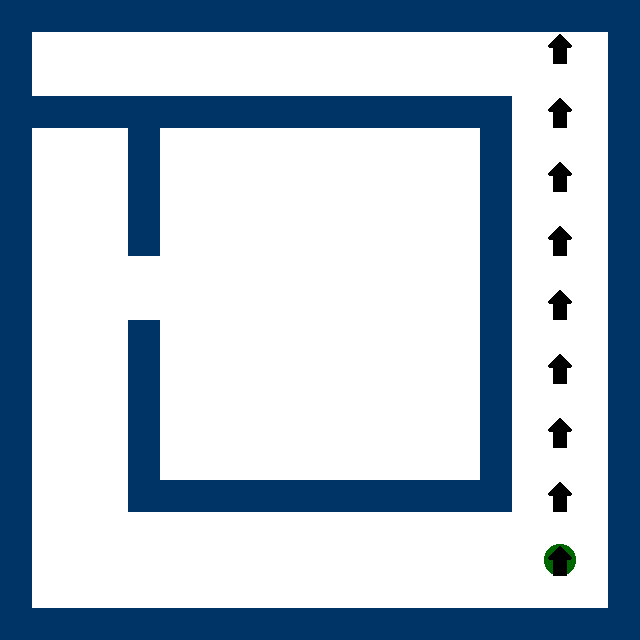}\hfill%
    \includegraphics[width=.5\linewidth]{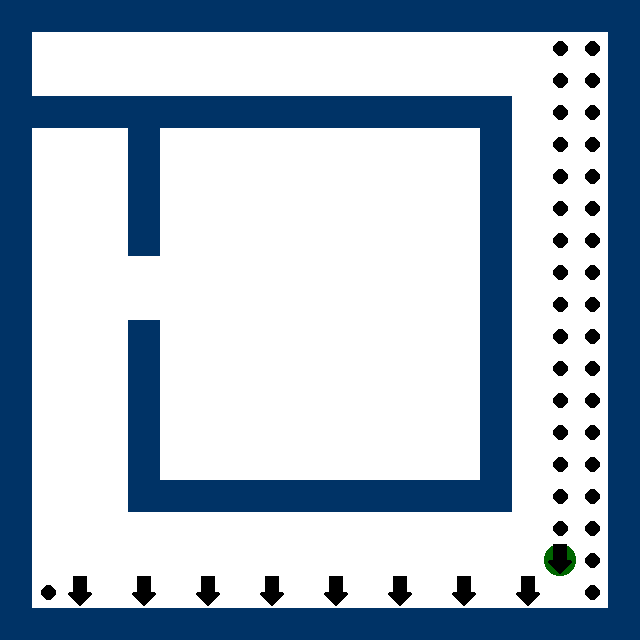}\hfill%
    \hfill
    \includegraphics[width=.5\linewidth]{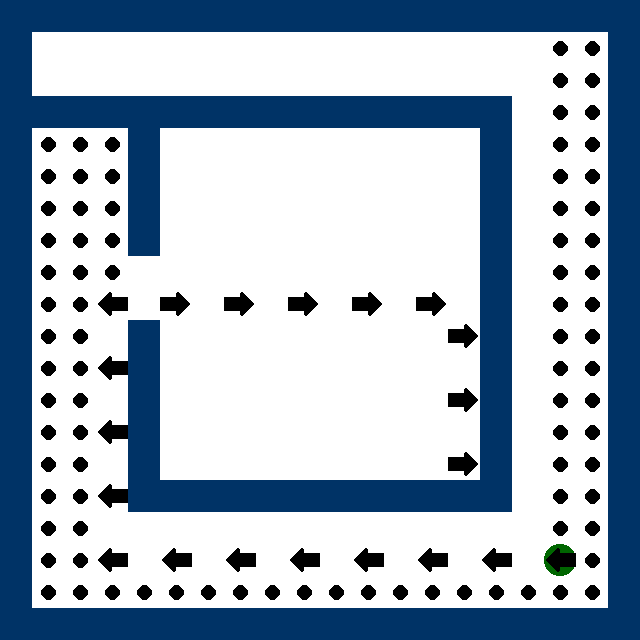}\hfill%
    \includegraphics[width=.5\linewidth]{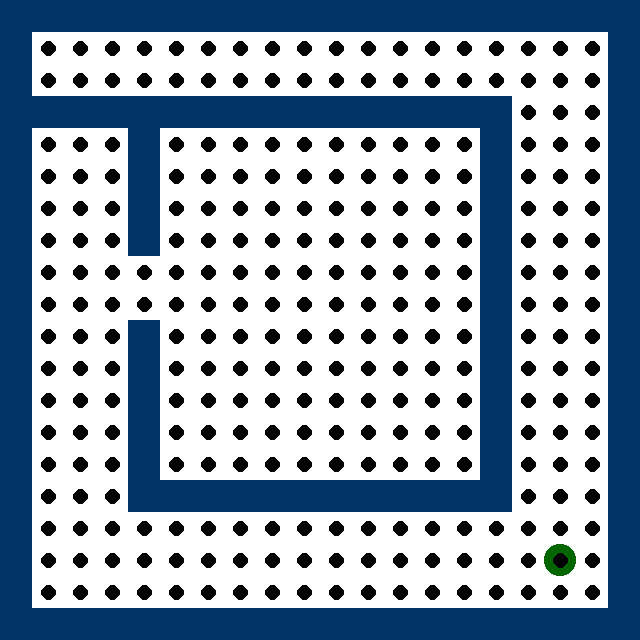}\hfill%
    \caption{Four snapshots from a simulation of FCDFS. \revisionhighlight{Arrows denote \textit{primary direction}.} Rather than block active robots, the settled robots form halls to enable the swarm to explore more of the environment.}
    \label{fig:simulation2}
\end{figure}

As in \cref{proposition:omniscientpathoptima}, let $R(t)$ be the initial environment $R$ after removing every vertex that is occupied by a \textit{settled} robot at time $t$. A robot running FCDFS at time $t$ is searching for the corners and halls of $R(t)$. However, since we assume no communication capabilities, robots are unable to distinguish between active robots, and walls or settled robots. Hence, it is important to design FCDFS so that a robot never misidentifies a corner of $R(t)$ as a hall due to an active robot (rather than a wall or a  settled robot) occupying the diagonal and being identified as an obstacle. For this purpose we enable our robots to remember their two previous locations. We will show that an active robot can occupy the diagonal of a corner at time $t$ if and only if its predecessor occupied this diagonal at time $t-2$, thereby allowing the predecessor to distinguish between 'real' and 'fake' halls.

Pseudo-code for FCDFS is given in \cref{alg:FCDFS}. The code outlines  the Compute step of the Look-Compute-Move cycle. Note that since our robots operate in Look-Compute-Move cycles, when the pseudocode says ``step,'' we mean that this is the step the robot will take during the Move phase, and not a movement that occurs during computation. In \cref{alg:FCDFS}, we denote by $prev(A)$ the position of robot $A$ at the beginning of the previous time step, and by $prev(prev(A))$ its position two time steps ago. We also denote by $next(A)$ its position at the beginning of the next time step. FCDFS can be implemented with just $5$ bits of persistent memory, but for the sake of a simpler presentation, the pseudocode in \cref{alg:FCDFS} does not go into the memory management aspect of FCDFS. The complete $5$-bit implementation of FCDFS, including memory management, is deferred to \cref{appendix:fcdfs5bit} (see \cref{alg:5bitFCDFS}). 


\begin{algorithm}[!htb]
  \caption{Find-Corner Depth-First Search (local rule for active robots)}
  \begin{algorithmic}
    \State Let $v$ be the current location of $A$.
    \If{all neighboring vertices of $v$ are occupied}
        \State Settle.
    \ElsIf{$A$ has never moved}\Comment{Initialization}
        \State Search clockwise, starting from the "up" direction, for an unoccupied vertex, and set primary direction to point to that vertex.
    \EndIf
    \If{The closest grid cell in $A$'s primary direction is empty}
        \State Move there.
    \ElsIf{The closest grid cell in $A$'s secondary direction is empty}
        \State Move there.
    \Else\Comment{We are at a corner or a hall.}
    \If{Three neighboring vertices of $v$ are occupied}
        \State Settle.\Comment{Corner with three walls.}
    \ElsIf{$prev(prev(A)) = diag(v)$ \textbf{or} $diag(v)$ is unoccupied} 
        \State Settle.
    \Else\Comment{We think we are at a hall.}
        \State Set primary direction to point to the neighbour of $v$ different from $prev(A)$.
        \State Move in the primary direction.
    \EndIf
    \EndIf
  \end{algorithmic}
  \label{alg:FCDFS}
\end{algorithm}

\subsection{Analysis}

In this section we analyze the FCDFS algorithm and prove its optimality.

\begin{lemma}
\label{removecornersimplyconnected}
Let $c$ be a corner of a simply connected region $R$. Then:

\begin{enumerate}[label=(\alph*)]
\item $R - c$ (the region $R$ after removing $c$) is simply connected.
\item For any two vertices $u, v$ in $R - c$, the distance between $u$ and $v$ is the same as in $R$.
\end{enumerate}
\end{lemma}

\begin{proof}
Removing $c$ does not affect connectedness, nor does it affect the distance from $u$ to $v$, as any path going through $c$ can instead go through $diag(c)$. Further, as $c$ is adjacent to two walls, no path in $R-c$ can surround it, so $R-c$ also remains simply connected.   
\end{proof}

An \textit{articulation point} (also known as a separation or cut vertex) is a vertex of a graph whose deletion increases the number of connected components of the graph (i.e., disconnects the graph) \cite{reinharddiestel2017}.

\begin{lemma}
\label{treestructure}
The halls of a simply connected region are articulation points.
\end{lemma}

\begin{proof}
Let $h$ be a hall of a simply connected region $R$. Suppose for contradiction that $h$ is not an articulation point, and let $u$ and $u'$ be the neighbors of $h$. Then there is a path from $u$ to $u'$ that does not pass through $h$. Let $P$ be this path, and let $P'$ be the path from $u$ to $u'$ that goes through $h$. 

When embedded in the plane in the usual way, $R$ is in particular a simply connected topological space. The hall $h$ is embedded onto a unit square, whose four corners each touch a wall: three touch the two walls adjacent to $h$, and the fourth touches $diag(h)$. Joined together to form a closed curve, the paths $P$ and $P'$ form a rectilinear polygon that must contain at least one corner of $h$ in its interior. Hence, the curve $PP'$ in $R$ contains a part of $R^c$. This is a contradiction to the assumption that $R$ is simply connected, as it implies $R^c$ has at least two disconnected components. (See Figure \ref{fig:hall_lemma}).\end{proof}

\begin{figure}[!ht]
  \centering%
    \includegraphics[width=.49\linewidth]{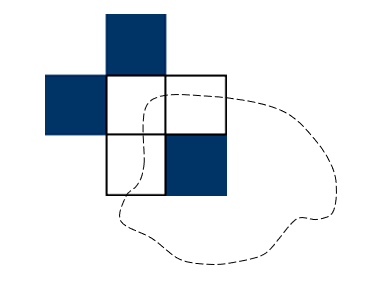}\hfill%
    \includegraphics[width=.49\linewidth]{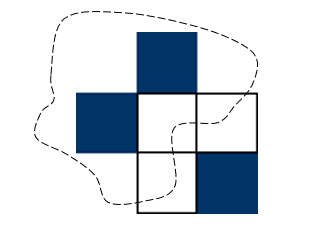}\hfill%
    \caption{The two possibilities for $PP'$.}
    \label{fig:hall_lemma}
\end{figure}

Lemma \ref{treestructure} indicates that $R$ can be decomposed into a tree structure $T(R)$ as follows: first, partition $R$ into the distinct connected components $C_1, C_2, \ldots, C_m$ that result from the deletion of all halls. Letting the vertices of $T(R)$ each represent one of these components (i.e., each vertex of $T(R)$ represents an entire connected component $C_i$ containing possibly many vertices of $R$, not individual vertices of $R$), connect $C_i$ and $C_j$ by an edge if they are both adjacent to the same hall in $R$. We set $C_1$ to be the root of the tree, and to be the connected component containing the door vertex $s$.

By Lemma \ref{removecornersimplyconnected}, assuming our robots correctly stop only at corners, $R(t)$ can in the same manner be decomposed into a tree $T(R(t))$ whose vertices represent connected components $C_1(t), C_2(t), \ldots$. These components each represent a sub-region of some connected component $C_i$ of $R$. Such tree decompositions will be important for our analysis.

In the next several propositions, we make the \textit{no fake halls at time $t$} assumption. This is the assumption that for any $t' < t$, at the end of time step $t'$: robots can only become settled at corners of $R(t')$, and can only change primary directions at halls of $R(t')$. We will later show that the ``no fake halls'' assumption is always true, so the propositions below hold unconditionally.

\begin{proposition}
\label{shortestpathproposition}
Assuming no fake halls at time $t$, a robot $A_i$ active at the beginning of time step $t$ has traveled an optimal path in $R$ from $s$ to its current position.
\end{proposition}

\begin{proof}
By the assumption, the only robots that became settled did so at corners. Consequently, by Lemma \ref{removecornersimplyconnected}, $R(t)$ is a connected graph, and there is a path in $R(t)$ from $s$ to $A_i$. The path $A_i$ took might not be in $R(t)$, but whatever articulation points (and in particular halls) $A_i$ passed through must still exist, by definition.

Since $A_i$ is active at the beginning of time $t$, by the algorithm, it has taken a step every unit of time up to $t$. Until $A_i$ enters its first hall, and between any two halls $A_i$ passes through, it only moves in its primary and secondary directions. This implies that the path $A_i$ takes between the halls of $R(t)$ must be optimal (since it is optimal when embedded onto the integer grid $\mathbb{Z}^2$). We note also that $A_i$ never returns to a hall $h$ it entered a connected component of $R(t)$ from, since the (possibly updated) primary direction pulls it away from $h$.

We conclude that $A_i$'s path consists of taking locally optimal paths to traverse the connected components of the tree $T(R(t))$ in order of increasing depth. Since in a tree there is only one path between the root and any vertex, this implies that $A_i$'s path to its current location is at least as good as the optimal path in $R(t)$. By Lemma \ref{removecornersimplyconnected}, b, this implies that $A_i$'s path is optimal in $R$.  
\end{proof}

\begin{corollary}
\label{distancecorollary}
Assuming no fake halls at time $t$,
\begin{enumerate}[label=(\alph*)]
\item For all $i < j$, the distance between the robots $A_i$ and $A_j$, if they are both active at the beginning of $t$, is at least $2(j-i)$
\item No collisions (two robots occupying the same vertex) have occurred. 
\end{enumerate}
\end{corollary}

\begin{proof}
For proof of (a), note that at least two units of time pass between every arrival of a new robot (since in the first time step after its arrival, a newly-arrived robot blocks $s$). Hence, when $A_j$ arrives, $A_i$ will have walked an optimal path towards its eventual location at time $t$, and it will be at a distance of $2(j-i)$ from $s$. This distance is never shortened up to time $t$, as $A_i$ will keep taking a shortest path. 

(b) follows immediately from (a).  
\end{proof}

Using Corollary \ref{distancecorollary} it is straightforward to show:

\begin{lemma}
\label{followtheleaderlemma}
Suppose $A_i$ is active at the beginning of time step $t$. Assuming no fake halls at time $t$, $next(A_{i+1}) = prev(A_i)$.
\end{lemma}

\cref{followtheleaderlemma} also implies that if $A_i$ is active at the beginning of time step $t$, then $A_{i+1}$ will be active at the beginning of time step $t+1$. 

We can now show that the ``no fake halls'' assumption is true, and consequently, the propositions above hold unconditionally.

\begin{proposition}
\label{nofakehalls}
For any $t$, at the end of time step $t$: robots only become settled at corners of $R(t)$, and only change primary directions halls of $R(t)$ (not including the primary direction decided at initialization).
\end{proposition}

\begin{proof}
The proof of the proposition is by induction. The base case for $t=1$ is trivially true.

Suppose that up to time $t-1$, the proposition holds. Note that this means the ``no fake halls'' assumption holds up to time $t$, so we can apply the lemmas and propositions above to the algorithm's configuration at the beginning of time $t$. 

We will show that the proposition statement also holds at time $t$. Let $A_i$ be an active robot whose location at the beginning of $t$ is $v$. First, consider the case where $v = s$. The algorithm only enables $A_i$ to settle at $s$ if it is surrounded by obstacles at all directions. Any obstacle adjacent to $A_i$ must be a wall of $R(t)$ (as any active robot must be at a distance at least $2$ from $A_i$, due to Corollary \ref{distancecorollary}). Hence, if $A_i$ settles at $s$, $s$ is necessarily a corner, as claimed.

We now assume $v \neq s$, separating the proof into two cases:

Case 1: Suppose $A_i$ becomes settled at the end of time step $t$. Then by the algorithm, at the beginning of $t$, $A_i$ detects obstacles in its primary and secondary directions. These must be walls of $R(t)$ due to Corollary \ref{distancecorollary}, so $v$ is either a corner or a hall of $R(t)$. If three neighbors of $v$ are occupied then $v$ is a corner. Otherwise, since $A_i$ settled, we further know that either $diag(v)$ is empty, or $prev(prev(A_i)) = diag(v)$. In the former case, $v$ is a corner of $R(t)$. In the latter case, we know from Lemma \ref{followtheleaderlemma} and from the fact that no collisions occur that the only obstacle detected at $diag(v)$ is $A_{i+1}$, which is an active robot, so $v$ is again a corner of $R(t)$. In either case a corner is detected and the agent is settled.

Case 2: Suppose $A_i$ changed directions at the end of time step $t$. Then it sees two adjacent obstacles, and an obstacle at $diag(v)$. As in case 1, we infer that $v$ is either a corner or a hall. If it is a corner, then $diag(v)$ is an active agent. By Corollary \ref{distancecorollary}, it is either $A_{i+1}$ or $A_{i-1}$. It cannot be $A_{i+1}$, as then $A_i$'s position two time steps ago would have been $diag(v)$, so it would become settled instead of changing directions. It cannot be $A_{i-1}$, as $diag(v)$ is at least as close to $s$ as $v$ is, and $A_{i-1}$ has arrived earlier than $A_i$, and has been taking a shortest path to its destination. Hence, $diag(v)$ cannot be an active agent, and $v$ must be a hall as claimed.  
\end{proof}

We have shown that the no fake halls assumption is justified at all times $t$, hence we can assume that the propositions introduced in this section hold unconditionally.

\begin{proposition}
\label{timeanalysis}
Let $V$ be the number of vertices of $R$. At the end of time step $2n$, every cell is occupied by a settled robot.
\end{proposition}

\begin{proof}
Propositions \ref{shortestpathproposition} and \ref{nofakehalls} imply that robots take a shortest path in $R$ to their destination. That means that as long as the destination of a robot is not $s$ itself, robots will step away from $s$ one unit of time after they arrive. Until then, this means that robots arrive at $s$ at rate one per two time steps (see \cref{fig:robot_arrival}).

\begin{figure}[!ht]
    \centering
    \includegraphics[width=.95\linewidth]{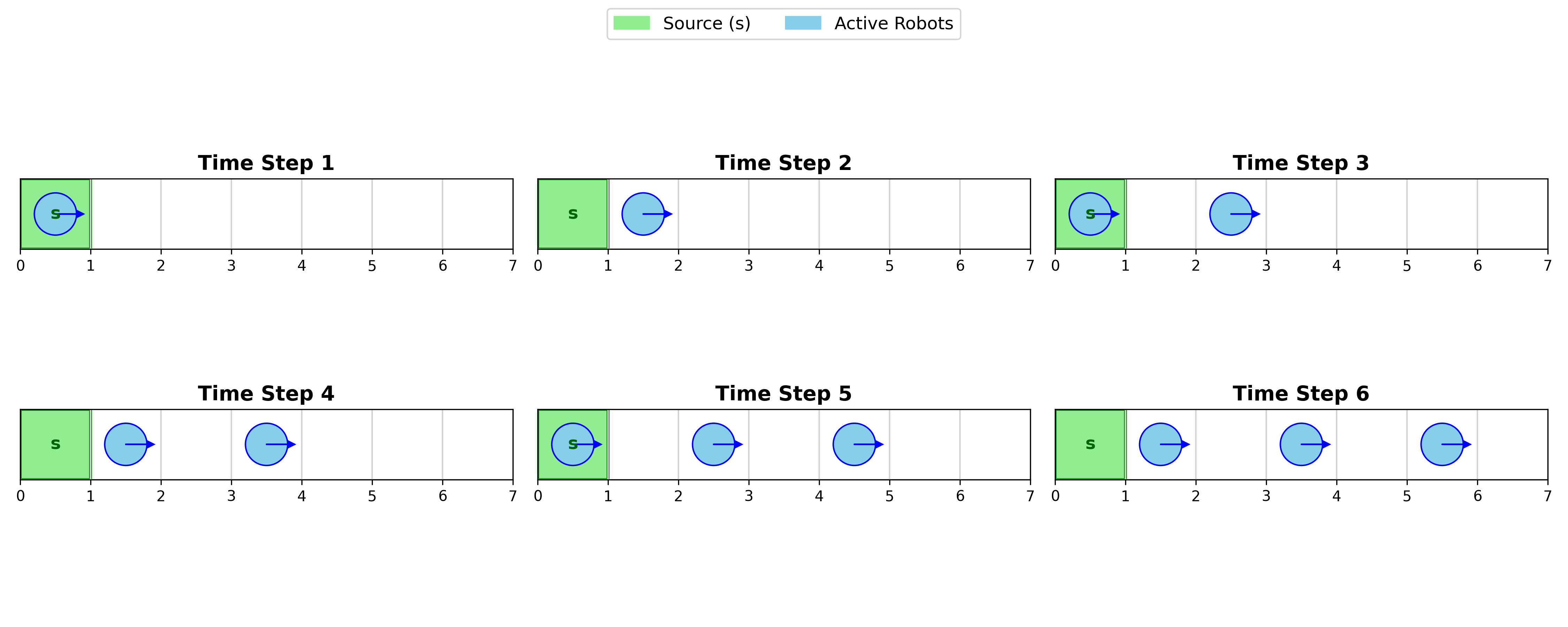}
    \caption{Illustration of robot arrival sequence at source $s$. At $t=1$, $A_1$ arrives at $s$. At $t=2$, $A_1$ moves away from $s$, freeing it up. At $t=3$, $A_2$ arrives at $s$ while $A_1$ continues forward. This pattern continues, resulting in robots arriving at $s$ at a rate of one per two time steps.}
    \label{fig:robot_arrival}
\end{figure}

Every robot's end-destination is a corner, and by the initialization phase of the algorithm, the destination is never $s$ unless $s$ is completely surrounded. Since there are no collisions, there can be at most $n$ robots in $R$ at any given time. By Lemma \ref{removecornersimplyconnected}, robots that stop at corners keep $R$ connected. Furthermore, every $R(t)$ is a rectilinear polygon, so unless it has exactly one vertex, it necessarily has at least two corners. This means that the destination of every robot is different from $s$ unless $s$ is the only unoccupied vertex. Hence, a robot whose destination is $s$ will only arrive when $s$ is the only unoccupied vertex, and this will happen when $n$ robots have arrived, that is, after at most $2n-1$ time steps. After another time step, this robot settles, giving us a makespan of $2n$. This is exact, since it is impossible to do better than $2n$ makespan.  
\end{proof}

Propositions \ref{timeanalysis} and \ref{shortestpathproposition}, alongside the ``no fake halls'' proof, complete our analysis. They show that FCDFS has an optimal makespan of $\mathcal{M} = \mathcal{M}^* = 2n$, and also that $E_{max} = E_{max}^* = 1 + \max_{v \in R}{\textrm{dist}(s,v)}$ and $E_{total} = E_{total}^* = n + \sum_{v \in R}{\textrm{dist}(s,v)}$, since every robot travels a shortest path to its destination without stopping. By the same reasoning, $T_{max} = T_{max}^* = \max_{v \in R}{\textrm{dist}(s,v)}$ and $T_{total} = T_{total}^* = \sum_{v \in R}{\textrm{dist}(s,v)}$.

In practice, the energy savings of FCDFS are dependent on the shape of the environment $R$. We take as a point of comparison the Depth-First Leader-Follower algorithm of Hsiang et al. \cite{hsiang}, described in \cref{prop:makespanminexists}. On a 1-dimensional path environment of length $n$, both FCDFS and DFLF require the same total travel, $\mathcal{O}(n^2)$, so no improvement is attained. In contrast, on an $n$-by-$n$ square grid, DFLF requires total travel $\mathcal{O}(n^4)$ in the worst case, and FCDFS requires $\mathcal{O}(n^3)$ - significantly less. This is because the DFLF strategy starting from a corner might cause the leader, $A_1$, to ``spiral'' inwards into the grid, covering every one of its $n^2$ vertices in $n^2 - 1$ moves. The robot $A_i$ will then make $n^2 - i$ moves, for a sum total of $O(n^4)$. FCDFS, on the other hand, distributes the path lengths more uniformly. More environments are compared in \cref{section:experiments}.  Note that both algorithms take the exact same amount of time to finish.

\textbf{Where is it best to place $s$?} If we want to minimize the total travel, by the formula given above, the best place to place $s$ is the vertex of $R$ that minimizes the sum of distances $\sum_{v \in R}{\textrm{dist}(s,v)}$ (there may be several). This is the discrete analogue of the so-called Fermat-Toricelli point, or the ``geometric median'' \cite{krarup1997torricelli}. 

\subsection{FCDFS Using 5 Bits of Memory}
\label{appendix:fcdfs5bit}

\begin{algorithm}[ht]
    \label{alg:5bitfcdfs}
  \caption{5-bit FCDFS}
  \begin{algorithmic}
    \State Let $v$ be the current location of $A$.
    
    \If{all neighboring vertices of $v$ are occupied}
        \State Settle.
    \ElsIf{$b_4b_5 = 00$}
        \State Search clockwise, starting from the "up" direction, for an unoccupied vertex, and set primary direction to point to that vertex.
        \State $b_4b_5 \gets 10$
    \EndIf
    
    \If{$A$ cannot move in primary or secondary directions}
        \If {$v$ has just one unoccupied neighbour}
            \State Settle.
        \ElsIf{$(b_5 = 1 \land b_3+b_4=1)$ $\lor$ $diag(v)$ is unoccupied} 
            \State Settle.
        \Else  
            \State Set primary direction to obstacle-less direction not equal to $180$-degree rotation of previous direction stepped in (i.e., the neighbour of $v$ we haven't visited yet; this can be inferred from $b_1b_2$ and $b_3$).  
            \State $b_4b_5 \gets 10$
        \EndIf
    \EndIf
    
    \If{$b_4b_5$ was not updated at this time step}\Comment{i.e., $b_5 = 1$ or time to update $b_5$}
        \State $b_4b_5 \gets b_3 1$
    \EndIf
    
    \If{$A$ can move in its primary direction}
        \State Move to the closest vertex in the primary direction.
        \State $b_3 \gets 0$
    \ElsIf{$A$ can move in secondary direction}
        \State Move to the closest vertex in the secondary direction.
        \State $b_3 \gets 1$
    \Else
        \State Settle.
    \EndIf
  \end{algorithmic}
  \label{alg:5bitFCDFS}
\end{algorithm}

As previously discussed, FCDFS can be implemented using 5 bit of memory: see \cref{alg:5bitFCDFS}. In this implementation, a robot's state is described by bits $b_1b_2b_3b_4b_5$. All bits are initially $0$. $b_1b_2$ describe the primary direction (one of four), and $b_3$ tells us whether the previous step was taken in the primary direction (if $b_3 = 0$) or in the secondary direction (if $b_3 = 1$). $b_4b_5$ is a counter that is reset to $10$ upon entering a hall or one step after initialization, and thereafter is equal to $*1$, where $*$ is a bit that tells us whether we walked in the primary or secondary direction two steps ago (by copying $b_3$). A robot that detects an obstacle at its diagonal interprets its position as a fake hall (i.e., a corner) as long as $b_5 = 1$ and $b_3 + b_4 = 1$, that is, as long as at least one step passed since the last hall, and our position two steps ago was diagonal to us.

\subsection{Alternate Optimal Strategies}

\begin{figure}[htb]
  \centering
  \begin{subfigure}[b]{0.48\columnwidth}
    \centering
    \includegraphics[width=0.8\linewidth]{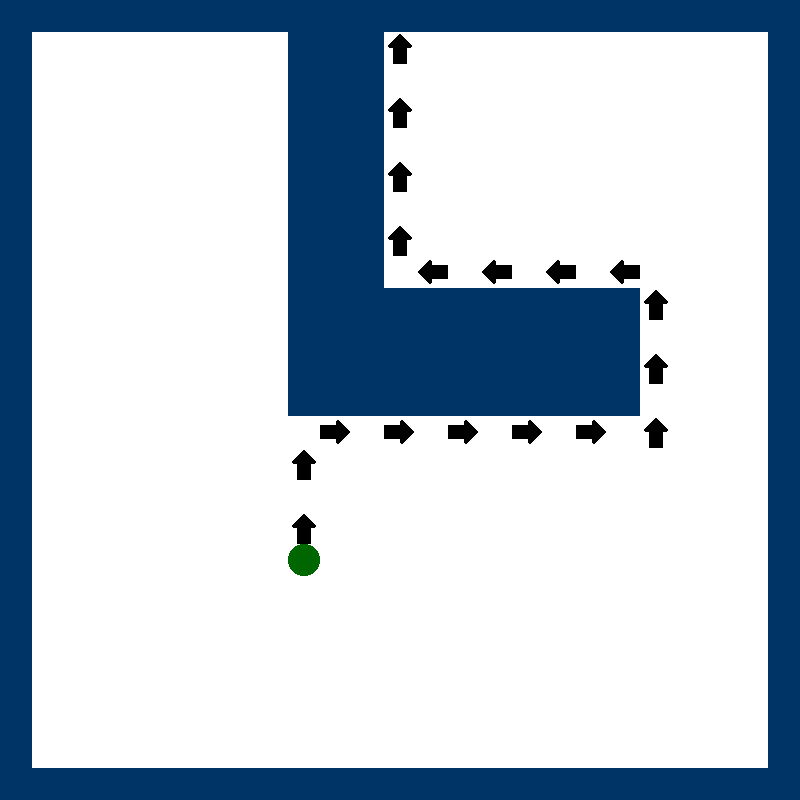}
    \caption{}
    \label{fig:alternatestrategies1a}
  \end{subfigure}\hfill%
  \begin{subfigure}[b]{0.48\columnwidth}
    \centering
    \includegraphics[width=0.8\linewidth]{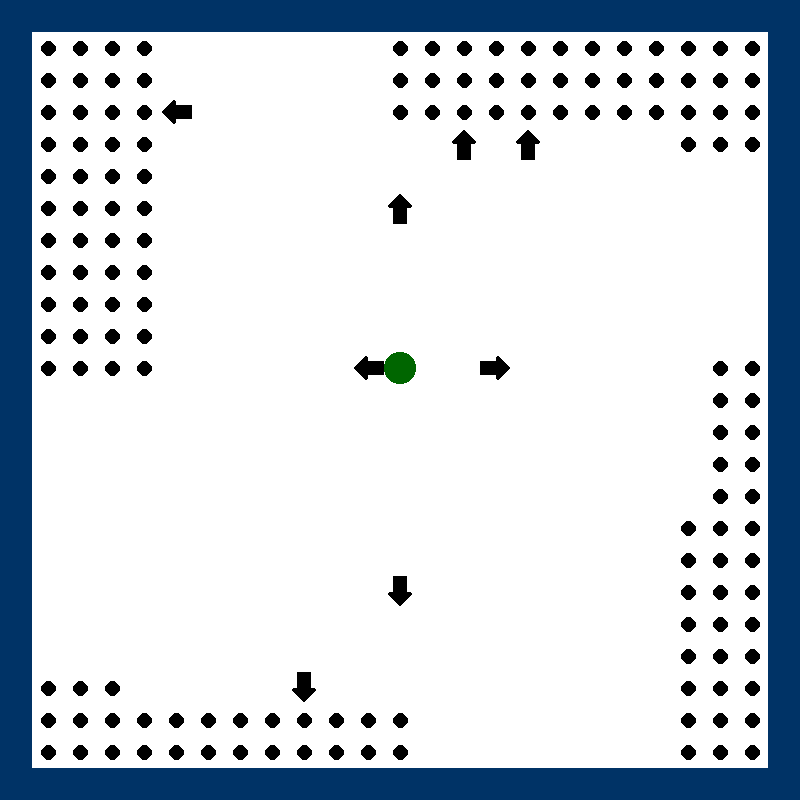}
    \caption{}
    \label{fig:alternatestrategies1b}
  \end{subfigure}
  \caption{Two alternate optimal dispersion strategies. \textbf{(a)} Robots navigate according to ``left-hand-on-wall'' until a corner is found. \revisionhighlight{\textbf{(b)} Robots stop early and are initialized with random orientations.}}
  \label{fig:alternatestrategies}
\end{figure}

The key idea of FCDFS to maintain a ``geometric invariant'': the environment must remain simply connected after a robot settles. As long as this invariant is maintained, the underlying details of the algorithm can vary. We depict here 
two possible variants of FCDFS that are similarly time- and energy-optimal  (analysis omitted). In \cref{fig:alternatestrategies1a}, rather than stick to their secondary and primary directions, robots move using a ``left hand on wall'' rule until they hit a corner or a wall. \revisionhighlight{In \cref{fig:alternatestrategies1a}, robots settle as soon as they reach a corner (whereas in standard FCDFS they keep moving if possible), and are initialized with random initial orientation. Unlike FCDFS, this variant does not require a compass/dead-reckoning.}

\section{Asynchronous FCDFS}
\label{section:asynchfcdfs}
In real-world multi-robot systems, each autonomous robot activates asynchronously, independent of other robots. One way to introduce asynchronicity into our system is to assume that at every time step, each robot has only a probability $p$ of waking up. When a robot wakes up, it performs the same Look-Compute-Move cycle as before. The source vertex $s$ wakes up at every time step with probability $p$, and upon wake-up inserts a robot into $R$, assuming no robot is currently located at $s$. Since we assume our robots communicate by light,  sound, or other local physical signals, we assume that the \textit{last} message broadcast by a robot $A_i$ continues to be broadcast in subsequent time steps, until $A_i$ broadcasts a new message. In other words, we assume that as long as $A_i$ is active, it  \textit{continuously} broadcasts its last message. This is akin to $A_i$ turning on a light or continuously emitting a sound.

These modifications, leaving everything else in our model the same (\cref{section:model}), result in an asynchronous model of swarms with local signalling. We wish to find a version of FCDFS that works in this asynchronous setting, and to study its performance. FCDFS as implemented in \cref{alg:FCDFS} relies crucially on the synchronous time scheme to identify halls. This is primarily because, assuming $0$ bits of communication bandwidth, robots in our model cannot tell the difference between obstacles (or settled robots) and active robots. In an asynchronous time scheme, the inability to recognize active robots can also cause robots to split away from the trail of robots that tends to forms under FCDFS (see \cref{fig:fcdfs_hall_example}). Both these issues can be fixed by enabling the robot executing FCDFS to detect active robots. This can be done by giving robots a single bit of communication bandwidth ($B = 1$) and having them broadcast `$1$' as long as they are active, indicating that they are active robots (in fact, this requires ``less'' than a bit of memory, since the robots never need to broadcast `$0$'). Note that, strictly speaking, such robots are no longer ant-like, but they still have very low capability requirements.

``AsynchFCDFS'' -- an asynchronous version of FCDFS incorporating the above idea -- is outlined in \cref{alg:asynchFCDFS}. It can be implemented as a $(2,1,5)$-algorithm, with the implementation similar to \cref{appendix:fcdfs5bit}). 

\begin{algorithm}[!htb]
  \caption{Asynchronous Find-Corner Depth-First Search (local rule for active robots)}
  \begin{algorithmic}
    \State Broadcast ``1''.\Comment{Inform other robots that you are active.}
    \State Let $v$ be the current location of $A$.
    \State Let $v_{primary}$ be the closest vertex in the primary direction of $A$.
    \State Let $v_{secondary}$ be the closest vertex in the secondary direction of $A$.
    \If{all neighbouring vertices of $v$ contain obstacles that are not robots broadcasting ``1''}
        \State Settle.
    \ElsIf{$A$ has never moved}\Comment{Initialization}
        \If{a neighbor of $v$ contains a robot broadcasting ``1''}
            \State Do nothing this time step.\Comment{Wait until active robots move away.}
        \Else
            \State Search clockwise, starting from the "up" direction, for an unoccupied vertex  and set primary direction to point to that vertex.
        \EndIf
    \EndIf
    \If{$v_{primary}$ contains a robot broadcasting ``1''}
        \State Do nothing this time step.
    \ElsIf{$v_{primary}$ is empty}
        \State Move there. 
    \ElsIf{$v_{secondary}$ contains a robot broadcasting ``1''}
        \State Do nothing this time step.
    \ElsIf{$v_{secondary}$ is empty}
        \State Move there. 
    \Else\Comment{We are at a corner or a hall.}
    \If{Three neighboring vertices of $v$ contain walls or settled robots}
        \State Settle.\Comment{Corner with three walls.}
    \ElsIf{$diag(v)$ is unoccupied or contains a robot broadcasting ``1''} 
        \State Settle.
    \Else\Comment{We think we are at a hall.}
        \State Set primary direction to point to the neighbour of $v$ different from $prev(A)$.
        \State Move to the closest vertex in the primary direction.
    \EndIf
    \EndIf
  \end{algorithmic}
  \label{alg:asynchFCDFS}
\end{algorithm}

AsynchFCDFS causes robots to \revisionhighlight{arrive at the same locations in the same order} they would under synchronous FCDFS, by requiring that robots wait in place if an active robot is occupying the direction they want to step in, until it leaves. \revisionhighlight{This means we retain FCDFS' no-collision guarantee (\cref{distancecorollary}) despite asynchronicity.}  It also tells us that \cref{alg:asynchFCDFS} is $T_{total}$ and $T_{max}$-optimal. Our main goal is thus to analyze the energy and makespan cost of waiting. Since our robots' activation times are now drawn from a random distribution, we will bound these costs probabilistically.

Studying the interactions of independently activating autonomous particles (in this case, modelling our robots) is generally considered a difficult and technical problem. Our key strategy is to side-step the difficulties of analyzing these interactions from scratch by relating a swarm of robots executing AsynchFCDFS to the ``Totally Asymmetric Simple Exclusion Process'' (TASEP) in statistical physics. Our high level strategy is to use \textit{coupling} to show that, in the worst case, AsynchFCDFS behaves similar to what is referred to as TASEP with step initial conditions, which is nowadays well-understood \cite{prahofer2002current,tracy2009asymptotics,chou2011biologicaltasep,chowdhury2000statistical,Johansson2000}. The main idea for this approach comes from \cite{amir_fast_2019}, in which the makespan of a different (``dual-layered'') uniform dispersion algorithm is studied using similar techniques. We generalize the techniques of \cite{amir_fast_2019} to study makespan and energy in our setting. 

Define the constant $\alpha = \frac{1}{2}(1 - \sqrt{1 - p})$. We show the following bound on makespan:

\begin{proposition}
The makespan of AsynchFCDFS over any region $R$ with $n$ vertices fulfills $\mathcal{M} \leq (\frac{1}{\alpha} + o(1))n$ asymptotically almost surely for $n \to \infty$ (i.e., with probability converging to $1$ as $n$ grows to infinity).
\label{prop:asynchfcdfsmakespan}
\end{proposition}

Setting $p = 1$ gives $\mathcal{M} = (2+o(1))n$,  asymptotically matching the synchronous case. Regarding energy consumption, we are able to show:

\begin{proposition}
In an execution of AsynchFCDFS over any region $R$ with $n$ vertices, $E_{max} \leq 2(\frac{1}{\alpha} + o(1))\max_{v \in R}{\textrm{dist}(s,v)}$ asymptotically almost surely for $n \to \infty$.
\label{prop:asynchfcdfsmaxenergy}
\end{proposition}

We shall also pose the following conjecture, based on a heuristic argument and  numerical simulations:

\begin{conjecture}

In an execution of AsynchFCDFS over any region $R$ with $n$ vertices, $E_{total} \leq (\frac{c}{\alpha}+o(1))\sum_{i=1}^{n}{\textrm{dist}(s,v_i)}$ asymptotically almost surely for $n \to \infty$ and some constant $c > 0$.


\label{conjecture:asynchfcdfsenergy}
\end{conjecture}

Since $p/\alpha \leq 4$ and $\mathcal{M}$ is bounded below by $n/(p-o(1))$ asymptotically almost surely (as robots are inserted at rate $p$), \cref{prop:asynchfcdfsmakespan} indicates that AsynchFCDFS's makespan is asymptotically  within a constant factor of optimal performance. Similarly, $E_{max} \geq \frac{1}{p-o(1)}\max_{v \in R}{\textrm{dist}(s,v)}$ a.a.s. (since robots move at rate $p$), so   \cref{prop:asynchfcdfsmaxenergy} implies AsynchFCDFS is asymptotically within a constant factor of optimal maximal individual energy use. \cref{conjecture:asynchfcdfsenergy} hypothesizes the same is true regarding total energy use. Hence, we believe AsynchFCDFS to be highly energy- and makespan-efficient in asynchronous settings, just as FCDFS is in synchronous settings.

\subsection{Analysis}

The execution of (synchronous) FCDFS over a region of interest $R$ is deterministic, and so we know in advance the path each robot will take. Let us define a graph, $G(R)$, whose nodes are the locations of $R$ and where there is an edge between two locations $v_i, v_j$ if at some point a robot moves from $v_i$ to $v_j$ during an execution of FCDFS. Recalling our analysis of FCDFS, it is not difficult to prove that $G(R)$ is a tree. Although robot wake-up times are random in our asynchronous model, the ultimate path each robot follows according to AsynchFCDFS is deterministic--the same path it would follow under FCDFS. Hence, in AsynchFCDFS, too, each robot moves down the edges of the tree $G(R)$.

To study the makespan of AsynchFCDFS, we use  coupling, drawing on a technique from  \cite{amir_fast_2019}. Suppose that there is an infinite collection of agents $A_1, A_2, \ldots A_\infty$, such that each agent $A_i$ wakes up with probability $p$ at every time step $t = 0, 1, \ldots$. We can associate with every agent $A_i$ a list of the times it wakes up, called $\mathcal{S}_i$. Suppose further that we execute FCDFS over different regions, say $R$ and $R'$. Let $A_i^R$ and $A_i^{R'}$ be the $i$th agents to enter $R$ and $R'$ respectively. We \textit{couple} $A_i^R$ and $A_i^{R'}$ by assuming that both $A_i^R$ and $A_i^{R'}$'s wake-up times are determined by $\mathcal{S}_i$. For formal purposes, when $A_i^R$ or $A_i^{R'}$ are not yet inside $R$, we still consider them to ``wake up'' at time steps $t \in \mathcal{S}_i$. This is a purely ``virtual'' wake-up: such robots do and affect nothing upon wake-up. The source vertices of $R$ and $R'$ can similarly be coupled.

Let $P_n$ be a ``straight path'' region consisting of the vertices $\{p_i \mid 0 \leq i < n\}$ where $p_i = (i, 0)$ and the source, $s$, is located at $p_0$. Furthermore, consider the stochastic process called TASEP with step initial conditions \cite{Johansson2000}, which consists of an infinite path region $P(\infty) = \{p_i \mid i \in (-\infty, \infty)\}$, where there is no source $s$ but instead the agent $A_i$ is initially located at $p_{-i}$. Finally, let $R$ be any region of interest consisting of $n$ vertices. To study the asymptotic performance of FCDFS over $R$, we shall couple the agents over TASEP, $P_n$ and $R$ (see \cref{fig:tasepcoupling}). Let us denote by $A_i^{\mathrm{TASEP}}$, $A_i^{P(n)}$, $A_i^R$,  the $i$th agent in each of these coupled processes, respectively.

\begin{figure}[ht]
    \centering
    \includegraphics[width=1\linewidth]{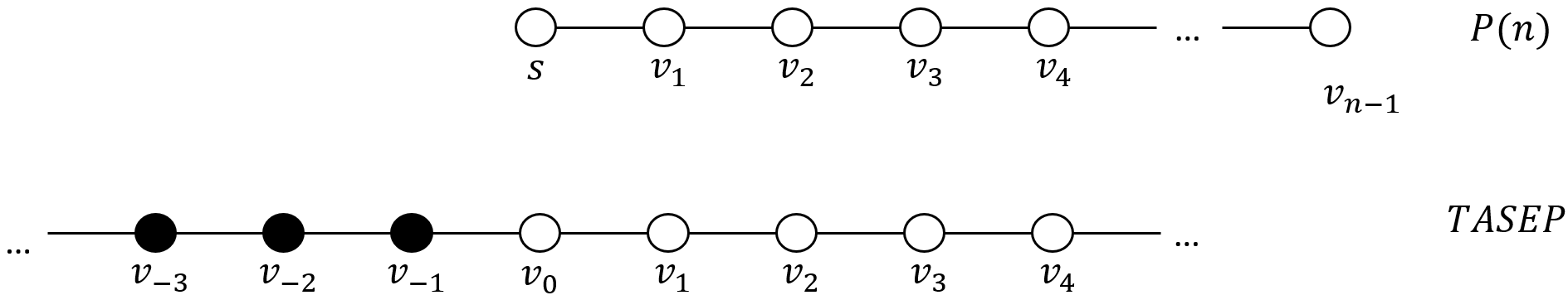}
    \caption{Illustration of $P(n)$ and TASEP with step initial conditions. Black vertices contain agents. Note that TASEP is infinite and lacks a source vertex.}
    \label{fig:tasepcoupling}
\end{figure}

We require a notion of the distance an agent has travelled on $R$ and $P(n)$:

\begin{definition}
    The \textbf{depth} at time $t$ of an agent $A_i^R$ executing FCDFS over a region $R$, denoted $depth_i^R(t)$, is the number of times $A_i$ has moved before time $t$. $depth_i^R(t) = -1$ if $A_i$ has not entered $R$ by time $t$, and $0$ at time of entry. We similarly define $depth_i^{P(n)}(t)$.
\end{definition}

Note that any robot that increases its depth is moving down the tree $G(R)$ from its current location $v \in G(R)$ to one of $v$'s descendants. Let us prove the following statement:

\begin{lemma}
    \label{lemma:depthbiggerinPnthanR}
    If the agent $A_i^R$ is not settled at time $t$ then:

    \begin{enumerate}
        \item $A_i^{P(n)}$ is not settled at time $t$, and
        \item $depth_i^R(t) \geq depth_i^{P(n)}(t)$.
    \end{enumerate}
\end{lemma}

\begin{proof}
We prove this lemma by induction over $t$. At $t = 0$, statements (1) and (2) are vacuously true. Now suppose we are at time $t+1$, and statements (1) and (2) hold at time $t$. Suppose for contradiction that statements (1) or (2) don't hold at time $t+1$. This means there exists $i$ such that $A_i^R$ is not settled at time $t+1$ and either (a) $A_i^{P(n)}$ is settled at time $t+1$, or (b) $depth_i^R(t+1) < depth_i^{P(n)}(t+1)$. 

Assume first, for contradiction, that (a) is true and (b) is not true. By the structure of $P(n)$, $A_i^{P(n)}$ being settled means that $depth_i^{P(n)}(t) = n-i-1$ and that (if $i > 1$) all of $v_{n-i}, \ldots v_{n-1}$ contain the settled robots $A_{1}^{P(n)}, \ldots A_{i-1}^{P(n)}$ at time $t$. By (1) and (2), this means that $A_{1}^{R}, \ldots A_{i-1}^{R}$ are settled at time $t$. Since $A_i^R$ is not settled at time $t+1$, this must mean that either it moved and increased its depth since time $t$, or it was adjacent an active agent in one of its primary directions at time $t$. The latter is impossible, since (by the FCDFS algorithm) such an agent must have been one of the agents $A_{1}^{R}, \ldots A_{i-1}^{R}$. Hence, $A_i^R$ must have moved at time $t$. By our assumptions, this implies  $depth_i^{R}(t+1) \geq depth_i^{P(n)}(t) + 1 = n-i$. But this is impossible, since there are at least $i-1$ settled agents in $R$ at time $t+1$, hence the maximum depth an active agent can have is $n-i-1$. Contradiction.

Next, let us assume for contradiction that (b) is true. By the inductive assumption, this implies that $depth_i^{P(n)}(t) = depth_i^{R}(t)$ and $depth_i^{P(n)}(t+1) = depth_i^{R}(t+1) + 1$, which can only occur if $A_i^{P(n)}$ moves at time $t$ and $A_i^R$ does not. $A_i^R$ not moving at time $t$ implies (due to the fact it is moving along the edges of the tree $G(R)$) that there is some $j < i$ such that $depth_j^R(t) = depth_i^R(t) + 1$. Since $A_i^{P(n)}$ does move at time $t$, by the structure of $P(n)$ we infer that $depth_j^{P(n)}(t) > depth_i^{P(n)}(t) + 1$. However, again by (2), at time $t$ we have $depth_j^R(t) \geq depth_j^{P(n)}(t)$, which implies $depth_j^R(t) > depth_i^{P(n)}(t) + 1 = depth_i^{R}(t) + 1$. Contradiction. 
\end{proof}

Let $x_i(t)$ be the $x$ coordinate of $A_i^{\mathrm{TASEP}}$'s location at time $t$ (where the $x$ coordinate of $v_i \in P(\infty)$ is $i$). It is straightforward, using coupling and induction as in \cref{lemma:depthbiggerinPnthanR}, to show that:

\begin{lemma}
    \label{lemma:depthbiggerinPnthanPinf}
    If the agent $A_i^{P(n)}$ is not settled at time $t$ then $depth_i^{P(n)}(t) \geq x_i(t)$.
\end{lemma}

\cref{lemma:depthbiggerinPnthanPinf} can also be rephrased as: if $depth_i^{P(n)}(t) < x_i(t)$ then $A_i^{P(n)}$ is settled. As a corollary, we infer:

\begin{corollary}
If at least $n+1$ TASEP agents have non-negative $x$ coordinate at time $t_0$ then the makespan of FCDFS over $R$ obeys $\mathcal{M} \leq t_0$.
\label{corollary:makespanofasynchfcdfs}
\end{corollary}

Note that here we are referring to the makespan of FCDFS over $R$ in given an arbitrary list of wake-up times. The list of wake-up times is randomly distributed (and so is  $\mathcal{M}$), but \cref{corollary:makespanofasynchfcdfs} is true for any such list.

\begin{proof}
Suppose for contradiction that some agent  $A_i^R$, $1 \leq i \leq n$, is not settled, but the $n+1$ agents $A_1^{\mathrm{TASEP}}, \ldots A_n^{\mathrm{TASEP}}$ have non-negative $x$ coordinate. By \cref{lemma:depthbiggerinPnthanR} we infer that $A_i^{P(n)}$ is not settled, and by \cref{lemma:depthbiggerinPnthanR} we infer $depth_i^{P(n)}(t) \geq x_i(t)$. However, the maximum possible depth of $A_i^{P(n)}$ is $n-i$, whereas $x_i(t)$ is at least $n-i+1$, since at least $n-i+1$ agents are located to its left and to the right of $v_0$. Contradiction.
\end{proof}

We can now prove \cref{prop:asynchfcdfsmakespan}. We use a similar idea as Lemma III. 14 of \cite{amir_fast_2019}.

\begin{proof} [Proof of \cref{prop:asynchfcdfsmakespan}] 
Let $F(t)$ denote the number of agents in TASEP that have non-negative $x$ coordinate at time $t$. It is known (see, e.g., \cite{Johansson2000,kriecherbauer2010pedestrian}) that $F(t)$ fulfills:

\begin{equation}
    \lim_{t \to \infty} \mathbb{P}(F(t) - \alpha t \leq Vt^{1/3}s) = 1 - TW_2(-s)
\label{equation:taseplimitingbehavior}
\end{equation}

where $V = 2^{-4/3}p^{1/3}(1 - p)^{1/6}$ and $TW_2(-s)$ is the Tracy-Widom distribution. $TW_2(-s)$ obeys the asymptotics $TW_2(-s) = O(e^{-c_1 s^3})$ and $1-TW_2(s) = O(e^{-c_2 s^{3/2}})$ as $s \to \infty$ for some $c_1, c_2 > 0$. 

Let $t_{k} = (1/\alpha+k^{-1/3})k$. We wish to show that $\mathcal{M} \leq t_{n+1}$ asymptotically almost surely as $n \to \infty$. By \cref{corollary:makespanofasynchfcdfs}, this is equivalent to showing that $X_{n} = \mathbb{P}(F(t_{n}) < n)$ tends to $0$ as $n \to \infty$ (because the probability that $\mathcal{M} \leq t$ is at least $1-X_{n+1}$). Define the probability

\begin{equation}
p(n,s) = \mathbb{P}(F(t_{n}) - \alpha t_{n} \leq Vt_{n}^{1/3}s)
\end{equation}

$p(n,s)$ is clearly monotonic non-decreasing in $s$. Define $s_n = (n-\alpha t_{n})V^{-1}t_{n}^{-1/3}$, yielding $X_{n} = p(n,s_n)$. For any constant $C$, define $Y_{n} = p(n,C)$. Since $s_n$ tends to $-\infty$ as $n \to \infty$, we must have $s_n < C$ for sufficiently large $n$, hence $X_n \leq Y_n$ for sufficiently large $n$. Using  \cref{equation:taseplimitingbehavior} and taking $C \to -\infty$ shows that $X_n$ tends to $1 - TW_2(\infty) = 0$ as $n \to \infty$, completing the proof.
\end{proof}

The proof of \cref{prop:asynchfcdfsmaxenergy} uses a similar technique to that of  \cref{prop:asynchfcdfsmakespan}, so we omit some repetitive details and focus on the new ideas:

\begin{proof}
[Proof sketch of \cref{prop:asynchfcdfsmaxenergy}] Let $d = \max_{v \in R}{\textrm{dist}(s,v)}$. Consider the execution of AsynchFCDFS over some region $R$ with $n$ vertices. The key idea of the proof is that at any given time, there can be no more than $d$ active robots in $R$. This is because all active agents always trace the path of the active agent $A_i$ with smallest index $i$ until it settles, and the path of $A_i$ is always a shortest path from $s$ to some vertex $v$. This path can contain at most $d$ vertices, and so there can be at most $d$ active agents. 

Consider a new agent $A_j$ that has emerged from $s$ at time step $t$. Since there can be at most $d$ active agents at a time, $A_j$ is necessarily settled once $d$ newer  agents have emerged from $s$. The time it takes $d$ agents to emerge after $A_j$ is dominated by the time it takes $2d$ agents to emerge in $R(t)$ (we can show this formally by the same type of coupling as \cref{lemma:depthbiggerinPnthanR}). As we've seen in the proof of \cref{prop:asynchfcdfsmakespan}, the time it takes $2d$ agents to emerge in $R(t)$ is dominated by the time it takes $2d$ agents to attain non-negative coordinates in TASEP, which we've previously shown is asymptotically almost surely less than $2(1/\alpha + o(1))d$ for large values of $d$. We complete the proof by noting that $d$ necessarily tends to $\infty$ as $n \to \infty$.
\end{proof}

We believe \cref{conjecture:asynchfcdfsenergy}, which states that $E_{total}$ is within a constant factor of optimal energy use, to be true, because every robot in AsynchFCDFS travels the same optimal path it does in FCDFS, and our results in this section hint that it may move at a linear rate toward this destination (indeed, \cref{prop:asynchfcdfsmaxenergy} establishes this for robots that travel distance $\max_{v \in R}{\textrm{dist}(s,v)}$). However, we struggled bounding $E_{total}$ with our current approach. The main challenge is that  \cref{equation:taseplimitingbehavior} is given as a limit, making it difficult to estimate the energy costs of agents that move short or intermediate distances. We pose this as a question for future work.

\section{Empirical Analysis}
\label{section:experiments}

To verify our theoretical findings \revisionhighlight{and measure performance metrics}, we \revisionhighlight{simulated} our algorithms on various simply-connected environments. \revisionhighlight{The simulator implements our model (Section II)}, and implements FCDFS and AsynchFCDFS using their minimal required capabilities ($(2,0,5)$ and $(2,1,5)$ respectively). To measure the improvement attained by FCDFS and AsynchFCDFS, we compared it to Hsiang et al.'s DFLF and BFLF algorithms \cite{hsiang}. DFLF is described in \cref{prop:makespanminexists}. BFLF (Breadth-First Leader-Follower) is a different makespan-optimal dispersion algorithm that seeks to spread out the agents more evenly by having multiple leaders, reducing total and maximal individual travel at the cost of energy. 

We tested our algorithms on several environment types: custom designed environments (Figures \ref{fig:fcdfs_hall_example} and \ref{fig:simulation2}), realistic environments from the "Benchmarks for Grid-Based Pathfinding" repository \cite{sturtevant2012benchmarks}, and systematically generated square grids of increasing size.  Table \ref{table:experiment_summary} summarizes simulation results across these environments. The benchmark environments, shown in \cref{fig:benchmarks}, include complex layouts and indoor urban environments that we selected and minimally modified to ensure simple-connectedness. This diverse set of environments allows us to validate our theoretical results under varying practical conditions.

\begin{figure*}[ht]
    \centering
    \begin{subfigure}[b]{0.15\textwidth}
        \includegraphics[width=\linewidth]{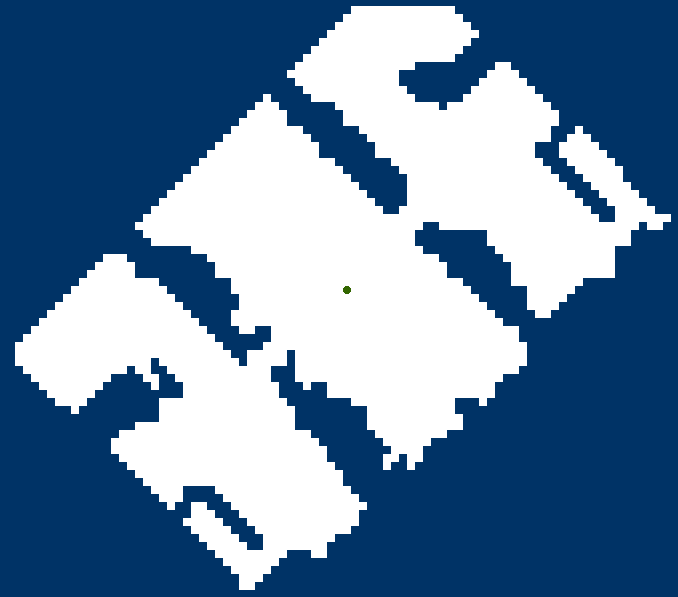}
        \caption{}
        \label{fig:AR0017SR}
    \end{subfigure}
    \hspace{0.02\textwidth} 
    \begin{subfigure}[b]{0.15\textwidth}
        \includegraphics[width=\linewidth]{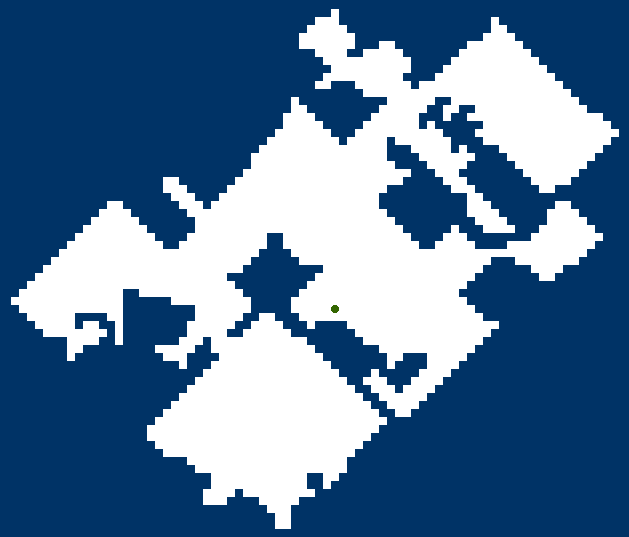}
        \caption{}
        \label{fig:AR0306SR}
    \end{subfigure}
    \hspace{0.02\textwidth} 
    \begin{subfigure}[b]{0.15\textwidth}
        \includegraphics[width=\linewidth]{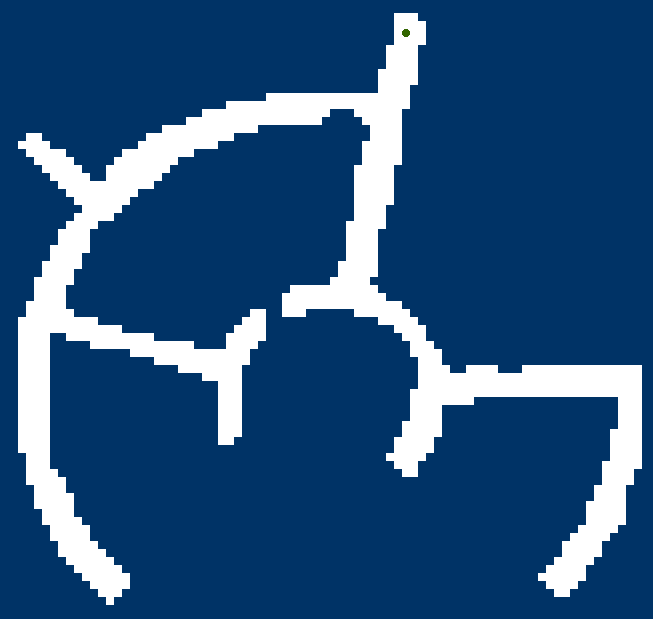}
        \caption{}
        \label{fig:AR0413SR}
    \end{subfigure}
    
    \begin{subfigure}[b]{0.15\textwidth}
        \includegraphics[width=\linewidth]{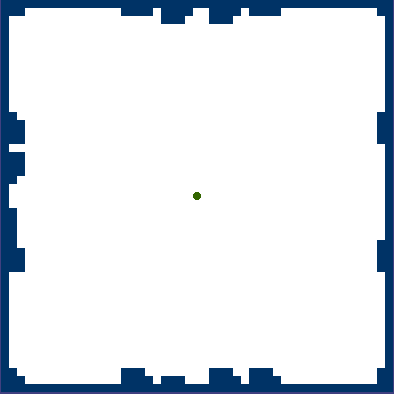}
        \caption{}
        \label{fig:arena}
    \end{subfigure}
    \hspace{0.02\textwidth} 
    \begin{subfigure}[b]{0.15\textwidth}
        \includegraphics[width=\linewidth]{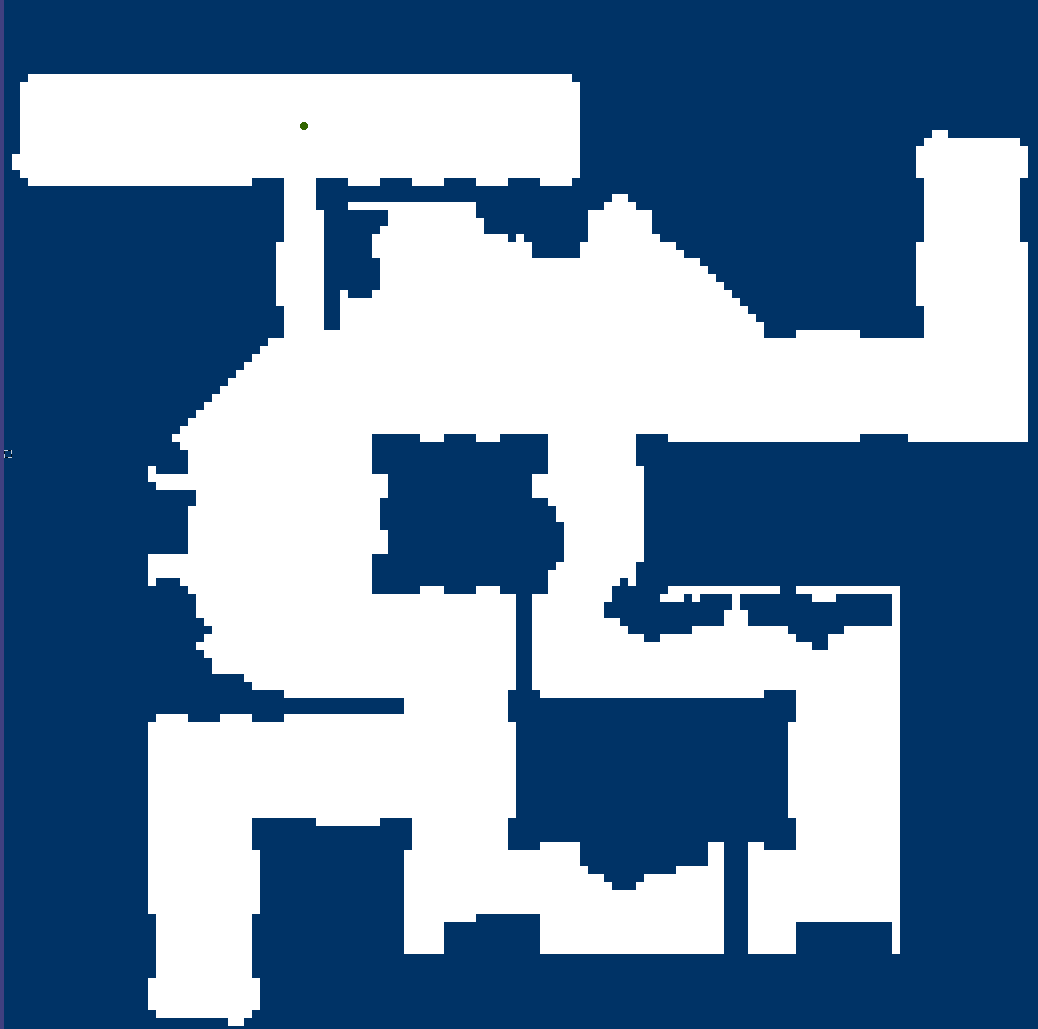}
        \caption{}
        \label{fig:lt_backvalley_n}
    \end{subfigure}
    \hspace{0.02\textwidth} 
    \begin{subfigure}[b]{0.15\textwidth}
        \includegraphics[width=\linewidth]{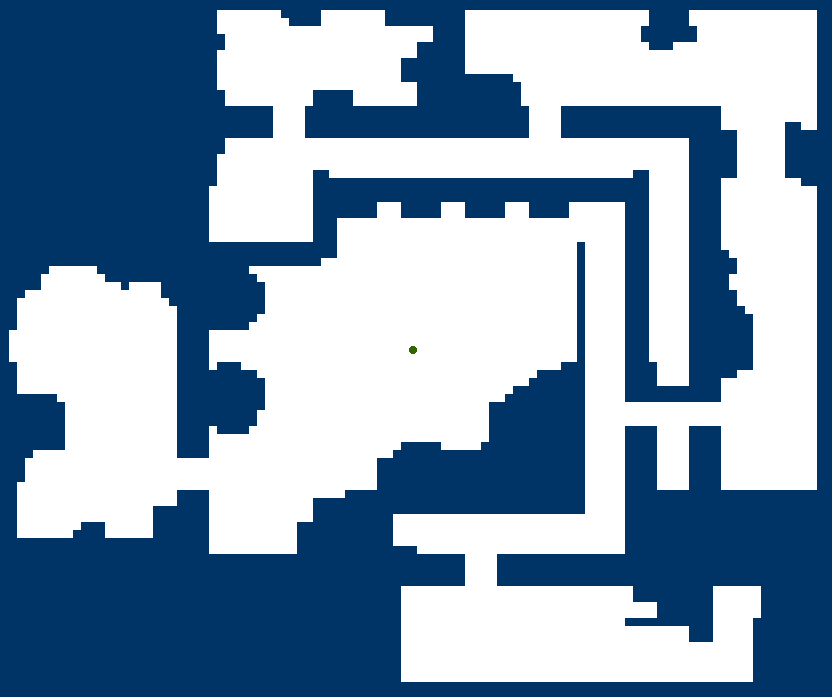}
        \caption{}
        \label{fig:foundry_n}
    \end{subfigure} 
    \caption{Benchmark environments based on \cite{sturtevant2012benchmarks}.}
    \label{fig:benchmarks}
\end{figure*}

As guaranteed by our theoretical results, FCDFS achieves optimal performance across all tested environments, including the realistic benchmarks. The performance gap between FCDFS and previous algorithms (DFLF, BFLF) remains consistent, and large,  across environment types, with no cases where previous methods outperform FCDFS in any of our metrics. However, while FCDFS achieves optimality in simply connected environments, DFLF and BFLF remain valuable options for environments with holes or obstacles that break simple connectedness.

Looking at specific performance patterns, we observe that the travel distance for FCDFS and AsynchFCDFS is identical, but makespan and energy use are larger for smaller $p$. For $p=0.5$ and $p=0.75$, AsynchFCDFS's average performance is consistent with the theoretical bounds of Propositions \ref{prop:asynchfcdfsmakespan} and \ref{prop:asynchfcdfsmaxenergy}, and \cref{conjecture:asynchfcdfsenergy}. While our theoretical analysis holds, provably, for all $p \in (0,1]$, these results provide further empirical validation of our claims.

Next, we were interested in how FCDFS and AsynchFCDFS's performance varies as $n$ (the environment size) grows. We plotted the performance of these algorithms alongside DFLF and BFLF on square grids of increasing size: see \cref{fig:comparative_results}. We see that, although the synchronous algorithms DFLF, BFLF and FCDFS are matched in terms of makespan, FCDFS attains significantly better energy and travel. AsynchFCDFS attains worse makespan, as expected (since robots are inactive in some timesteps), but still attains better energy and travel than DFLF and BFLF for $p = 0.5, 0.75$. AsynchFCDFS's maximal individual energy use and makespan are consistent with the theoretical bounds given by Propositions \ref{prop:asynchfcdfsmakespan} and \ref{prop:asynchfcdfsmaxenergy}, and total energy use consistent with \cref{conjecture:asynchfcdfsenergy}.

\begin{table*}[ht]
\centering
\caption{Simulation results across various environments. Metrics represent averages over 10 trials. Abbreviations: 'k' indicates thousands, and 'AsFCDFS' refers to Asynchronous FCDFS. Standard deviations are indicated by ±. If no ± is shown, the standard deviation is 0.}
\label{table:experiment_summary}
\begin{scriptsize}
\resizebox{0.70\textwidth}{!}{%
\begin{tabular}{@{}l l c c c c c@{}}
\toprule
\textbf{Environment} & \textbf{Algorithm} & \textbf{$T_{\text{total}}$} & \textbf{$T_{\text{max}}$} & \textbf{$E_{\text{total}}$} & \textbf{$E_{\text{max}}$} & \textbf{$\mathcal{M}$} \\
\midrule
\multirow{5}{*}{\cref{fig:AR0017SR} ($n = 2501$)} & FCDFS & 78k & 61 & 80k & 62 & 5002 \\
 & AsFCDFS (p=0.75) & 78k & 61 & 153k $ \pm 2k$ & 94 $ \pm 17$ & 9480 $ \pm 47$ \\
 & AsFCDFS (p=0.5) & 78k & 61 & 261k $ \pm 4k$ & 97 $ \pm 33$ & 16034 $ \pm 91$ \\
 & BFLF & 90k $ \pm 2k$ & 82 $ \pm 7$ & 1406k $ \pm 100k$ & 269 $ \pm 470$ & 5002 \\
 & DFLF & 895k $ \pm 127k$ & 647 $ \pm 110$ & 897k $ \pm 127k$ & 648 $ \pm 110$ & 5002 \\
\midrule
\multirow{5}{*}{\cref{fig:AR0306SR} ($n = 1846$)} & FCDFS & 70k & 69 & 72k & 70 & 3692 \\
 & AsFCDFS (p=0.75) & 70k & 69 & 140k $ \pm 2k$ & 88 $ \pm 12$ & 7073 $ \pm 41$ \\
 & AsFCDFS (p=0.5) & 70k & 69 & 236k $ \pm 3k$ & 122 $ \pm 41$ & 11922 $ \pm 64$ \\
 & BFLF & 89k $ \pm 5k$ & 94 $ \pm 9$ & 691k $ \pm 40k$ & 276 $ \pm 239$ & 3692 \\
 & DFLF & 457k $ \pm 81k$ & 479 $ \pm 59$ & 459k $ \pm 81k$ & 480 $ \pm 59$ & 3692 \\
\midrule
\multirow{5}{*}{\cref{fig:AR0413SR} ($n = 1006$)} & FCDFS & 67k & 124 & 68k & 125 & 2012 \\
 & AsFCDFS (p=0.75) & 67k & 124 & 131k $ \pm 2k$ & 186 $ \pm 27$ & 3899 $ \pm 17$ \\
 & AsFCDFS (p=0.5) & 67k & 124 & 223k $ \pm 3k$ & 241 $ \pm 76$ & 6603 $ \pm 56$ \\
 & BFLF & 72k $ \pm 2k$ & 135 $ \pm 3$ & 312k $ \pm 12k$ & 200 $ \pm 50$ & 2012 \\
 & DFLF & 138k $ \pm 6k$ & 252 $ \pm 15$ & 139k $ \pm 6k$ & 253 $ \pm 15$ & 2012 \\
\midrule
\multirow{5}{*}{\cref{fig:arena} ($n = 2122$)} & FCDFS & 49k & 45 & 51k & 46 & 4244 \\
 & AsFCDFS (p=0.75) & 49k & 45 & 97k $ \pm 872$ & 60 $ \pm 9$ & 7940 $ \pm 49$ \\
 & AsFCDFS (p=0.5) & 49k & 45 & 164k $ \pm 2k$ & 77 $ \pm 30$ & 13428 $ \pm 75$ \\
 & BFLF & 61k $ \pm 3k$ & 74 $ \pm 12$ & 1525k $ \pm 68k$ & 417 $ \pm 660$ & 4244 \\
 & DFLF & 1186k $ \pm 98k$ & 961 $ \pm 78$ & 1188k $ \pm 98k$ & 962 $ \pm 78$ & 4244 \\
\midrule
\multirow{5}{*}{\cref{fig:lt_backvalley_n} ($n = 6862$)} & FCDFS & 579k & 177 & 586k & 178 & 13724 \\
 & AsFCDFS (p=0.75) & 579k & 177 & 1143k $ \pm 7k$ & 250 $ \pm 45$ & 26782 $ \pm 57$ \\
 & AsFCDFS (p=0.5) & 579k & 177 & 1941k $ \pm 14k$ & 382 $ \pm 143$ & 45528 $ \pm 115$ \\
 & BFLF & 706k $ \pm 22k$ & 215 $ \pm 14$ & 6594k $ \pm 395k$ & 521 $ \pm 225$ & 13724 \\
 & DFLF & 5640k $ \pm 665k$ & 1542 $ \pm 174$ & 5647k $ \pm 665k$ & 1543 $ \pm 174$ & 13724 \\
\midrule
\multirow{5}{*}{\cref{fig:foundry_n} ($n = 4552$)} & FCDFS & 404k & 208 & 408k & 209 & 9104 \\
 & AsFCDFS (p=0.75) & 404k & 208 & 799k $ \pm 4k$ & 260 $ \pm 59$ & 17652 $ \pm 53$ \\
 & AsFCDFS (p=0.5) & 404k & 208 & 1363k $ \pm 10k$ & 478 $ \pm 133$ & 29958 $ \pm 58$ \\
 & BFLF & 506k $ \pm 25k$ & 271 $ \pm 8$ & 2671k $ \pm 127k$ & 368 $ \pm 65$ & 9104 \\
 & DFLF & 2845k $ \pm 656k$ & 1204 $ \pm 233$ & 2850k $ \pm 656k$ & 1205 $ \pm 233$ & 9104 \\
\midrule
\multirow{5}{*}{\cref{fig:fcdfs_hall_example} ($n = 669$)} & FCDFS & 35k & 99 & 36k & 100 & 1338 \\
 & AsFCDFS (p=0.75) & 35k & 99 & 69k $ \pm 1k$ & 144 $ \pm 35$ & 2570 $ \pm 19$ \\
 & AsFCDFS (p=0.5) & 35k & 99 & 117k $ \pm 3k$ & 199 $ \pm 75$ & 4332 $ \pm 53$ \\
 & BFLF & 40k $ \pm 1k$ & 117 $ \pm 8$ & 126k $ \pm 8k$ & 186 $ \pm 61$ & 1338 \\
 & DFLF & 98k $ \pm 11k$ & 276 $ \pm 34$ & 98k $ \pm 11k$ & 277 $ \pm 34$ & 1338 \\
\midrule
\multirow{5}{*}{\cref{fig:simulation2} ($n = 277$)} & FCDFS & 6k & 38 & 6k & 39 & 554 \\
 & AsFCDFS (p=0.75) & 6k & 38 & 12k $ \pm 414$ & 50 $ \pm 6$ & 1028 $ \pm 25$ \\
 & AsFCDFS (p=0.5) & 6k & 38 & 20k $ \pm 826$ & 71 $ \pm 25$ & 1716 $ \pm 38$ \\
 & BFLF & 8k $ \pm 364$ & 51 $ \pm 4$ & 33k $ \pm 2k$ & 120 $ \pm 34$ & 554  \\
 & DFLF & 17k $ \pm 892$ & 129 $ \pm 10$ & 17k $ \pm 892$ & 130 $ \pm 10$ & 554 \\
\bottomrule
\end{tabular}
}
\end{scriptsize}
\end{table*}

\begin{figure*}[ht]
    \centering
    \begin{subfigure}[b]{0.3\textwidth}
        \includegraphics[width=\linewidth]{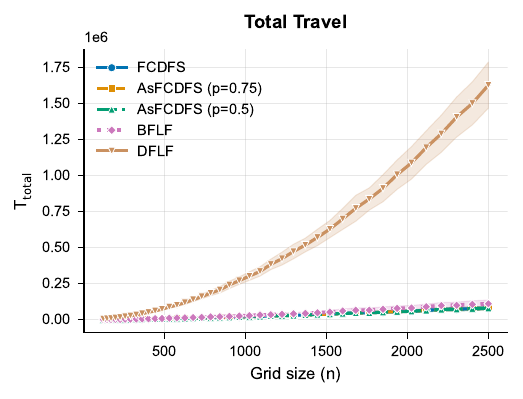}
    \end{subfigure}
    \hfill
    \begin{subfigure}[b]{0.3\textwidth}
        \includegraphics[width=\linewidth]{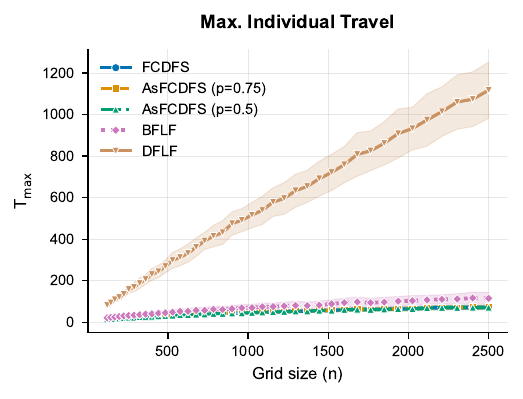}
    \end{subfigure}
    \hfill
    \begin{subfigure}[b]{0.3\textwidth}
        \includegraphics[width=\linewidth]{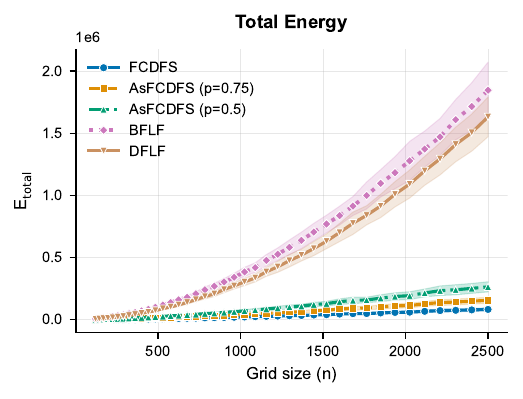}
    \end{subfigure}
    
    \medskip 
    
    \begin{subfigure}[b]{0.3\textwidth}
        \includegraphics[width=\linewidth]{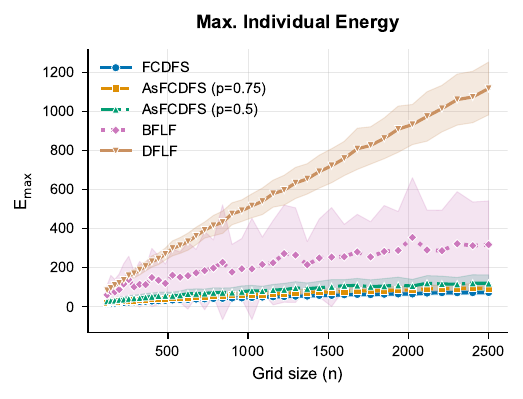}
    \end{subfigure}
    \quad 
    \begin{subfigure}[b]{0.3\textwidth}
        \includegraphics[width=\linewidth]{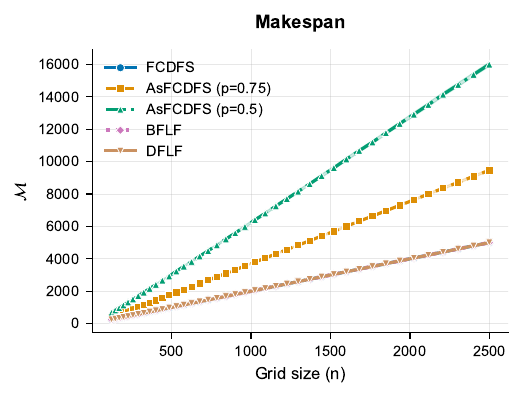}
    \end{subfigure}
    
    \caption{Comparative results of the algorithms across different metrics. For each $k \in [10,50]$, we generated $100$ regular square grid environments of size $n = k \times k$ with different randomly placed source locations. Each data point is the average of $100$ simulations of its respective algorithm, one per such grid environment. In each graph, the $x$ axis is $n$ and the $y$ axis is the  performance measure of interest. Shaded regions denote standard deviations. Note that some algorithms overlap perfectly on metrics such as makespan.}
    \label{fig:comparative_results}
\end{figure*}

\section{Conclusion}

This work studied time, travel, and energy use in the uniform dispersion problem. We presented a formal model for comparing swarm algorithms in terms of robot capabilities, and proved that while makespan and travel can be minimized in all environments by swarm robotic algorithms with low capability requirements, energy cannot be minimized even by a centralized algorithm, assuming constant sensing range. In contrast, we showed that energy can be minimized in simply connected environments by an algorithm called FCDFS, executable even by minimal ``ant-like'' robots with small constant sensing range and memory, and no communication. We extended FCDFS to the asynchronous setting, showing it attains energy consumption within a constant factor of the optimum.

Our results shed light on the relationships between the capabilities of robots, their environment, and the quality of the collective behavior they can achieve. The fact that robots can minimize the travel or makespan even with very limited capabilities, but cannot minimize energy even with arbitrarily strong capabilities, attests to the difficulty of designing energy-efficient collective behaviors. On the other hand, the fact that energy can be minimized by ant-like robots in simply connected environments suggests energy efficiency in swarm robotic systems may be practically attainable, especially in restricted settings. Broadly, our results point to the importance of \textit{co-design} approaches in swarm robotics, carefully tailoring the capabilities of robots to the specific features of the environment they are expected to operate in.

The capability-based modeling approach we presented for uniform dispersion can be applied to other canonical problems in swarm robotics, such as  gathering, formation, or path planning. Doing so may reveal additional gaps in capability requirements  between robots optimizing different objectives, as well as new co-design opportunities.

\revisionhighlight{As our goal in this work was a formal mathematical analysis of the \textit{algorithmic} components of time, travel, and energy use, we assumed an idealized grid setting. However, in the real world, energy is shaped by physical realities beyond our grid model--e.g. motion uncertainty, kinematic constraints, and environmental factors such as small debris, floor friction, and slopes. The problem of handling such empirical factors, and superimposing them top of our theoretical lower bounds, is an important direction for future work.}


\section*{Acknowledgements}

The authors would like to thank Ofer Zeitouni (Weizmann Institute of Science) for helpful discussions.

\appendix
\label{appendix:}
\subsection{Ant-like Robots are ``Strongly'' Decentralized}
\label{appendix:stronglydecentralizedantrobots}

\cref{proposition:centralizedsimul} establishes that decentralized algorithms, given sufficient memory and communication capabilities, can emulate centralized algorithms while incurring only a multiplicative time delay ($\Delta$), provided their communication network remains connected. Thus, when decentralized robots collectively gather the same environmental data as their centralized counterpart, they can achieve equivalent decision-making capabilities.  We claim that ant-like algorithms--algorithms requiring small,  \textit{constant} visibility and persistent state memory, and no communication, to run on their target class of environments--are ``strongly decentralized'': these algorithms, despite potentially sensing the same collective environmental data, cannot replicate certain decisions made by centralized algorithms. We prove this by giving an example of a centralized, visibility preserving algorithm that ant-like algorithms cannot simulate, even assuming arbitrary $\Delta$-delay.

Let $P_n$ be a ``straight path'' region consisting of the vertices $\{p_i \mid 0 \leq i < n\}$ where $p_i = (i, 0)$ and the source, $s$, is located at $p_0$. Let $\textrm{ALG}(i,n)$ be an algorithm that acts as follows on the region $P_n$ (its target environments): at every time step, it tells every robot to step right (i.e., from $p_k$ to $p_{k+1}$), until $A_1$ arrives at $p_i$, at which point the algorithm permanently halts. At halting time,  this algorithm results in a chain of robots, the leftmost of which is either at $p_1$ or $p_0$, such that $A_i$ and $A_{i+1}$ are two steps apart. For all $i$, $n$, $\textrm{ALG}(i,n)$ is easily executable over $P_n$ by a centralized robotic swarm with visibiltiy $V = 2$ and $S = 0$ bits of persistent memory (the centralized algorithm can know where $A_1$ is at every time step, hence whether $A_1$ is located at $p_i$ at a given time step, by simply counting the number of robots in the environment).

It is clear that $\mathrm{ALG}(i,n)$ is $2$-visibility preserving. A $\Delta$-delay version of it can be executed by decentralized robots with communication and sufficient memory, per \cref{proposition:centralizedsimul}. However, ant-like robots cannot execute $\textrm{ALG}(i,n)$ for sufficiently large $i$ and $n$, even at a delay:

\begin{proposition}
For all integers $\Delta \in [1, \infty)$, no $(v,0,s)$-algorithm exists that simulates $\textrm{ALG}(2^s + v + 2, 2^s + 2v + 2)$ over $P_{2^s + 2v + 2}$ at delay $\Delta$.
\label{prop:antlikerobotsarestronglydecentralized}
\end{proposition}

\begin{proof}
Suppose for contradiction that $\textrm{ALG}$ is a $(v,0,s)$-algorithm simulating $\textrm{ALG}(2^s + v + 2, 2^s + 2v + 4)$ over $P_{2^s + 2v + 4}$ at delay $\Delta$. $A_1$ moves rightward at every time step. At time steps $i \in [\Delta(v+1), \Delta(2^s+v+2)]$, it always sees the exact same thing: $v$ locations to its left, half of which contain an agent $A_i$ and half of which are empty, and $v$ empty locations to its right. Since $\textrm{ALG}$ has $2^s$ distinct memory states, at two distinct time steps $\Delta(v + 1) \leq t_1 <  t_2 \leq \Delta(2^s+v+2)$ it must have had the same memory state. Hence, at time steps $\Delta(2^s + v + 2)$ and $\Delta(2^s + v + 2 + (t_1 - t_2))$, $A_1$ has the same memory state and sees the same things. Since $(v,0,s)$ robots cannot communicate, this means these time steps are indistinguishable to $A_1$. Since, when executing $\textrm{ALG}$, $A_1$ does not step right at time step $\Delta (2^s + v + 2)$ (i.e., upon arriving at $p_{2^s + v + 2}$), it must also not do so at time step $\Delta (2^s + v + 2 + (t_1 - t_2))$.  But $\textrm{ALG}(2^s + v + 2, 2^s + 2v + 2)$ necessitates that $A_1$ step right at every time step $\Delta t$ for $t < 2^s + v + 2$ - contradiction.
\end{proof}

\cref{prop:antlikerobotsarestronglydecentralized} shows that for any ant-like swarm with constant capabilities $(v,0,s)$, there exist $n$ and $i$ such that $ALG(n,i)$ cannot be executed over a large  path graph. This gives evidence that ant-like robots are  ``strongly'' decentralized: they are fundamentally incapable of making the same decisions as a centralized algorithm. 

\bibliographystyle{ieeetr}  
\bibliography{references}  

\begin{IEEEbiography}[{\includegraphics[width=1in,height=1.25in,clip,keepaspectratio]{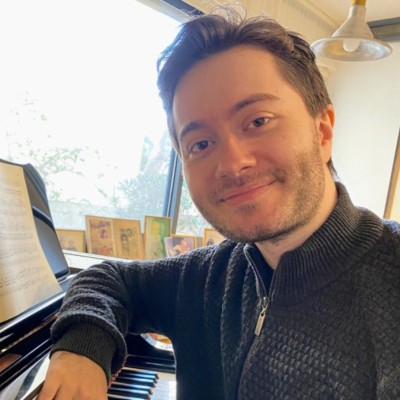}}]{Michael Amir}
 is a Research Associate at the University of Cambridge, affiliated with the Prorok Lab and Trinity College. He received his Ph.D.\ in Computer Science from the Technion, focusing on the mathematical analysis of natural and robotic swarms, where he was advised by Prof.\ Alfred  M.\ Bruckstein. 
His research interests include multi-robot and multi-agent systems, provable performance guarantees, and AI safety.
\end{IEEEbiography}

\begin{IEEEbiography}[{\includegraphics[width=1in,height=1.25in,clip,keepaspectratio]{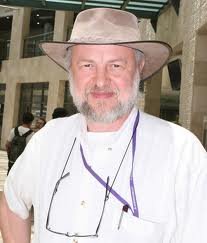}}]{Alfred M. Bruckstein} (Fellow, IEEE) 
 is a Visiting Professor at NTU, Singapore, and an Emeritus Ollendorff Professor of Science at the Technion. He was elected a SIAM Fellow for contributions to signal processing, image analysis, and ant robotics; an IEEE Fellow for work on image representation and swarm robotics; and a Fellow of the CORE Academy of Science and Humanities, among several other honors recognizing his research. His students hold key academic and research positions worldwide.\end{IEEEbiography}

\end{document}